\documentclass[letterpaper, 10 pt, conference]{ieeeconf}  

\IEEEoverridecommandlockouts                              

\usepackage{graphics} 
\usepackage{epsfig} 
\usepackage{amsmath} 
\usepackage{amssymb}  
\usepackage{algorithm}
\usepackage{algpseudocode}
\usepackage{mathtools}
\usepackage{graphicx,subfigure}
\usepackage[usenames,dvipsnames]{color}
\usepackage[nospace]{cite}

\graphicspath {{figures_final/}}

\newtheorem{theorem}{\bf{Theorem}}

\newtheorem{proposition}[theorem]{\bf{Proposition}}

\newtheorem{remark}[theorem]{\bf{Remark}}

\title{\LARGE \bf Multiscale Abstraction, Planning and Control using Diffusion Wavelets \\for Stochastic Optimal Control Problems}

	\author{Jung-Su Ha and Han-Lim Choi
		\thanks{J.-S. Ha and H.-L. Choi are with the Dept. of Aerospace Engineering, KAIST,  Korea       {\tt\small \{wjdtn1404, hanlimc\}@kaist.ac.kr}}%
}

\begin{document}

\maketitle
\thispagestyle{empty}
\pagestyle{empty}

\begin{abstract}
This work presents a multiscale framework to solve a class of stochastic optimal control problems in the context of robot motion planning and control in a complex environment.
In order to handle complications resulting from a large decision space and complex environmental geometry, two key concepts are adopted: (a) a diffusion wavelet representation of the Markov chain for hierarchical abstraction of the state space; and (b) a desirability function-based representation of the Markov decision process (MDP) to efficiently calculate the optimal policy.
In the proposed framework, a global plan that compressively takes into account the long time/length-scale state transition is first obtained by approximately solving an MDP whose desirability function is represented by coarse scale bases in the hierarchical abstraction.
Then, a detailed local plan is computed by solving an MDP that considers wavelet bases associated with a focused region of the state space, guided by the global plan.
The resulting multiscale plan is utilized to finally compute a continuous-time optimal control policy within a receding horizon implementation.
Two numerical examples are presented to demonstrate the applicability and validity of the proposed approach.

\end{abstract}

\section{Introduction}
In this work, we address the continuous time/continuous state stochastic optimal control (SOC) problem for robots operating in a complex environment over a long time horizon.
The SOC problem involves how one computes an optimal policy for a system that is driven by uncertain disturbances, to maximize a certain performance index.
The standard and general way of obtaining an optimal control solution is to use the dynamic programming approach, which  computes the optimal cost-to-go function (also called the value function) for all possible states and times, and then reconstruct a control policy from the value function.
In a continuous-time/continuous-state problem, a nonlinear partial differential equation called the Hamilton-Jacobi-Bellman (HJB) equation needs to be solved, which is computationally intractable in most robotic control applications.

There is a class of SOC problem, called the \textit{linearly-solvable optimal control (LSOC)}, where the HJB equation can be naturally linearized, and efficient solution methods exist \cite{kappen2005path,todorov2009efficient}.
In the LSOC, the exponentiated value function, termed the \textit{desirability function}, is obtained as the principal eigenfunction of the linearized differential operator.
Because of the linear structure in the LSOC, several effective solution schemes have been proposed.
These include: (a) approaches which approximate the eigenstructure with some generic function approximation techniques such as Gaussian radial basis function (RBF) \cite{todorov2009eigenfunction}; or (b) which obtain the derivative of desirability functions by sampling many continuous trajectories and evaluating them.
The latter method is called path integral control \cite{kappen2005path}.
However, if the LSOC problem of interest is associated with  a very long time horizon and/or complex/high-dimensional geometric domains (e.g., induced by obstacles), the aforementioned techniques cannot be readily implemented, since obtaining an appropriate basis set for function approximation becomes non-trivial, and simulating trajectories while checking potential collisions becomes computationally expensive \cite{todorov2009eigenfunction,ha2016topology,williams2016aggressive}.
If the original continuous SOC problem is discretized, it can be represented as a Markov decision process (MDP) and solved using well-developed iterative methods, such as policy iteration and value iteration.
However, this discretization approach is inherently hobbled by the curse of dimensionality, which imposes a significant limit when handling high-dimensional problems.
There is a discrete-time equivalent of the LSOC, called the linearly-solvable MDP (LMDP), but the solution methodologies for this class of problem still exhibit limitations when handling high-dimensional and long-horizon problems.

To address a large-scale problem effectively, it may be fruitful to take note of certain human intelligence processess - in particular, the \textit{multiscale} and \textit{hierarchical} structure of human decision making.
Suppose that someone currently writing a paper at their office desk wants to get out of the building.
Let's assume that; this third floor office is in a building with one staircases and an elevator.
Then, what would this person's control policy look like?
Unlike the standard value function-based approach, this person would not try to figure out what they should do for all possible situations they might face; instead, they would determine which exit from the room they would use (if there is more than one), whether to take the elevator or the stairs, which building gate they would use, etc. A detailed plan such as ``which foot should be used to start walking down the stairs,'' would be determined later in the process of executing a segment of the overall plan, for example, ``go downstairs using the staircase.''
It should be noted that this human decision making process takes advantage of the underlying multiscale and hierarchical structure of state space.
In the above example, details such as the particular individual sequence of steps on stairs are abstracted to just a single notion of ``using the staircase".

There have been various studies whose goal has been to determine the hierarchical structure of spaces.
One such approach which has been extensively studied is multi-resolution analysis (MRA) \cite{goswami2011fundamentals}.
Based on the wavelet theory, MRA attempts to find a sequence of basis sets that spans a certain nested subspace.
For example, when applied to a function approximation problem, the basis functions in a coarse length scale will only reconstruct the rough trend of the target function.
Although this MRA method does provide a systematic scheme for hierarchical abstraction, constructing the wavelet bases over a high-dimensional complex geometric domain is not straightforward.
The notion of a diffusion wavelet  \cite{coifman2006diffusion} has been shown to provide a more general wavelet bases construction procedure.
This type of abstraction over a Markov chain has been utilized to expedite policy evaluation for MDPs \cite{maggioni2006fast}.
Also, diffusion wavelets are used to construct the basis set in representation policy iteration \cite{mahadevan2005value}, which simultaneously learns the policy and the system representation during the executed task.

\begin{figure}[pt]
	\centering
	\includegraphics*[width=1\columnwidth, viewport=40 95 920 530]{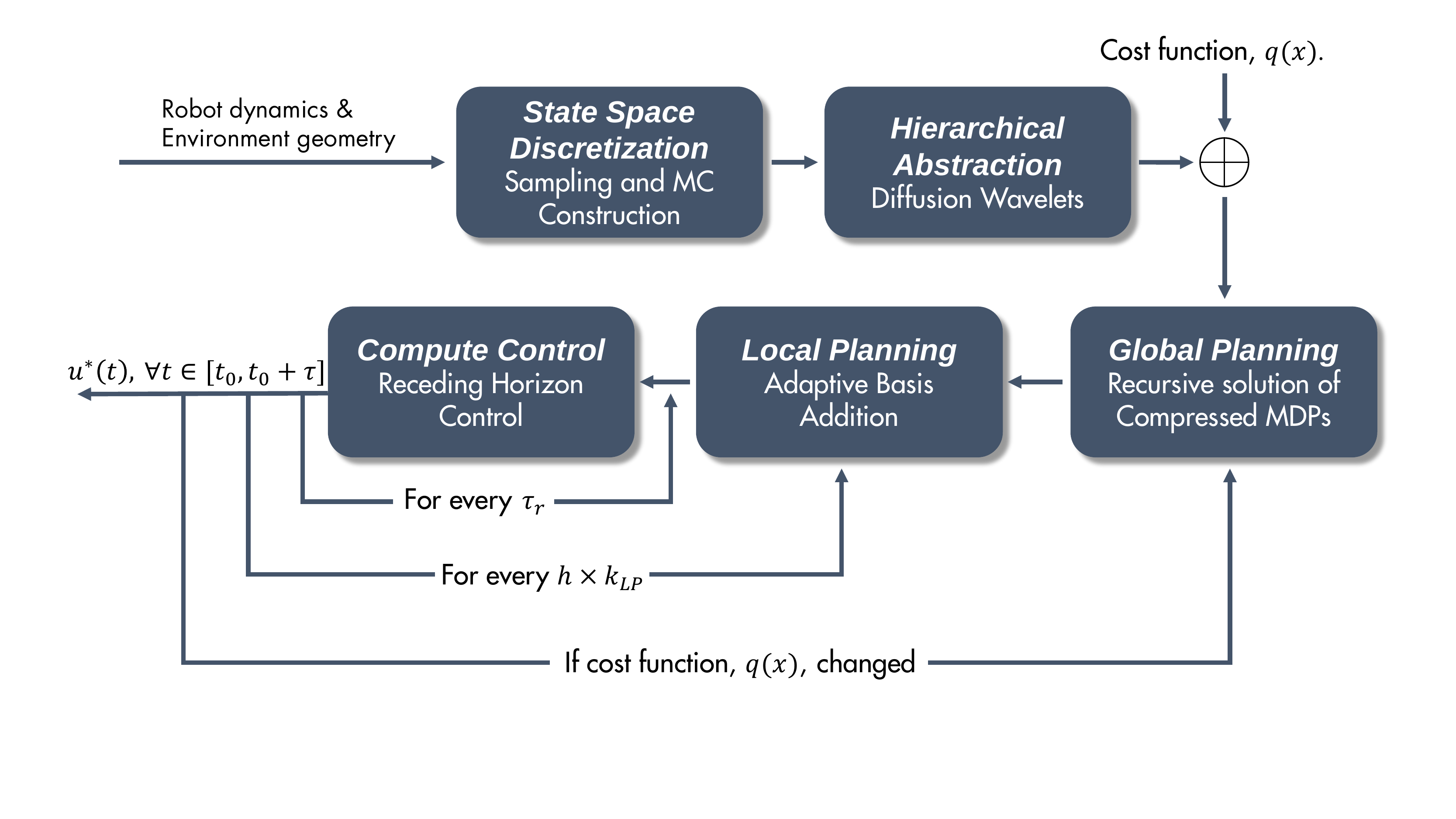}
	\caption{Proposed framework}
	\label{fig:framework}
	\vspace*{-.25in}
\end{figure}
This work addresses a large-scale, long time-horizon LSOC problem, taking advantage of an abstraction scheme using diffusion wavelet bases to approximate the desirability functions that naturally lead to a solution of the HJB equation. The key contribution of this paper is to present a systematic framework for solving LSOC, as depicted in Fig. \ref{fig:framework}. This is, to the authors' best knowledge, the first work that takes advantage of the multiscale structure of the basis set to solve an LSOC problem. The framework consists of five phases. (i)  In the discretization phase, the Markov chain associated with the robot dynamics is constructed by sampling a finite set of states in state-space. (ii) In the abstraction phase, the hierarchical bases structure is obtained using the diffusion wavelet method. (iii) In the global planning phase, an MDP constructed using only the coarse wavelet bases (or, on ``abstract-states'') is solved; this MDP is much more tractable than using the original basis set. Our work fundamentally differs from those of Mahadevan et al. and Maggioni et al. \cite{mahadevan2005value,maggioni2006fast} in the sense that the abstraction is obtained by discretizing the given stochastic system dynamics, and the multiscale structure of the basis sets is utilized to recursively solve compressed problems. (iv) In the local planning phase, focused regions where the robot will most likely visit in the near future are sought, and the detailed policy associated with these focused regions is computed. A certain search procedure is also applied to the coarse bases using optimal transition information which is naturally provided by the (approximate) solution of the LMDP. (v) In the control phase, a continuous control sequence is computed and applied to the robot in a receding horizon fashion. The remainder of the paper is primarily focused on elaborating the details of this framework, followed by numerical examples to validate the method.

\section{Continuous-time Stochastic Optimal Control Problem and Time Discretization}
Let $\mathbf{x}\in\mathcal{X}$ and $\mathbf{u}\in \mathcal{U}$ be the state and control vectors, respectively, where the state space, $\mathcal{X}$, and control input space, $\mathcal{U}$, are subsets of $\mathbb{R}^{d_x}$ and $\mathbb{R}^{d_u}$, respectively.
Suppose $\mathbf{w}$ is a $d_u$-dimensional Brownian motion process.
Consider the stochastic dynamics in which the deterministic drift term is affine in the control input:
\begin{equation}
d\mathbf{x} =  \mathbf{f}(\mathbf{x})dt+ G(\mathbf{x})(\mathbf{u}dt + \sigma d\mathbf{w}) \label{eq:conti_dyn}
\end{equation}
where $\mathbf{f}:\mathcal{X}\rightarrow \mathbb{R}^{d_x}$ is the passive dynamics and $G:\mathcal{X}\rightarrow \mathbb{R}^{d_x\times d_u}$ is the control transition matrix function.
Let a function $q: \mathcal{X} \rightarrow \mathbb{R}$ be an instantaneous state cost rate.
Then, the cost functional which we want to minimize is defined as:
\begin{equation}
J(\mathbf{x}) = \lim_{t_f\rightarrow\infty}\frac{1}{t_f}E\left[\int^{t_f}_0 q(\mathbf{x}(t))+\frac{1}{2\sigma^2}\mathbf{u}(t)^T\mathbf{u}(t)dt\right]. \label{eq:conti_cost}
\end{equation}
The problem with the cost function (\ref{eq:conti_cost}) and dynamics (\ref{eq:conti_dyn}) is called the infinite horizon average cost stochastic optimal control (SOC) problem.

The continuous-time problem is hard to solve except in some special cases.
In general, time-axis and state space are discretized to make the problem tractable.
Here, we introduce a time-axis discretized problem with a time step $h$.
The transition probability of one step without any control input is defined as:
\begin{equation}
\mathbf{x}[k+1] \sim p(\cdot|\mathbf{x}[k]), \label{eq:discrete_dyn_pas}
\end{equation}
which is called the passive dynamics.
If the control input is applied, the transition probability is changed and written as:
\begin{equation}
\mathbf{x}[k+1] \sim \pi(\cdot|\mathbf{x}[k]). \label{eq:discrete_dyn}
\end{equation}
The passive and controlled dynamics can be approximated as $\mathcal{N}(\mathbf{y}; \mu(h), \Sigma(h))$, where $\mathcal{N}$ is a Gaussian distribution with a mean $\mu(h)$ and covariance $\Sigma(h)$.
Also, for small $h$, the Kullback–-Leibler divergence between two distributions is approximated as $D_{KL}(\pi(\cdot|\mathbf{x})||p(\cdot|\mathbf{x})) = \frac{h}{2\sigma^2}\mathbf{u}'\mathbf{u}$.
Therefore, the cost functional (\ref{eq:conti_cost}) is written in discrete time setting:
\begin{align}
&J(\mathbf{x}) =\nonumber\\
&\lim_{K\rightarrow\infty}\frac{1}{K}E\left[\sum^{K}_{k=0} hq(\mathbf{x}[k])+D_{KL}\left(\pi(\cdot|\mathbf{x}[k])||p(\cdot|\mathbf{x}[k])\right)\right]. \label{eq:discrete_cost}
\end{align}
It is well known that the solution of the discrete-time SOC (\ref{eq:discrete_dyn_pas})-(\ref{eq:discrete_cost}) converges to the solution of the continuous-time SOC (\ref{eq:conti_dyn})-(\ref{eq:conti_cost}) as $h\rightarrow0$ \cite{todorov2009efficient}.

\section{State Space Discretization and Multi-Scale Abstraction}
\subsection{State Space Discretization and Associated Markov Chain}\label{sec:Graph}
In this subsection, the methods for state space discretization and the associated Markov chain construction will be discussed.
One general approach is to discretize the state space by grid:
this is very intuitive and simple, but it can easily suffer from the curse of dimensionality and hard to adapt the complex geometry of domains induced by obstacles.
Another approach to discretization is the sampling method.
The sampling method has several advantages over the grid method;
it is easy to control the number of discrete points (samples) and to utilize heuristics methods by adapting sampling density.
Also, collision checking modules are easily incorporated into the discretization scheme, which is the reason that the sampling-based algorithms are extensively studied and used in the motion planning literature.

Suppose a set of discrete states $X=\{\mathbf{x}_n\}$ is given.
Then the transition matrix for the passive dynamics $P$, where $P_{nm}$ means a transition probability from $\mathbf{x}_n$ to $\mathbf{x}_m$ should be determined such that they satisfy the following local consistency condition~\cite{kushner2013numerical,todorov2009efficient}:
\begin{align}
&\mu(h)\approx\bar{\mathbf{y}}_n=\sum_mP_{nm}\mathbf{x}_m, \\
&\Sigma(h)\approx\sum_mP_{nm}(\mathbf{x}_m-\bar{\mathbf{y}}_n)(\mathbf{x}_m-\bar{\mathbf{y}}_n)^T, \\
&\sum_mP_{nm}=1.
\end{align}
This requires solving the linear equation with the $(d_x^2+d_x+1)\times |X|$ matrix for each state.

Another approach for determining $P$ is approximation via Gaussian distribution.
Let our discrete sample distribution be given by $\Psi(\mathbf{x})$.\footnote{If the distribution is not known explicitly, a kernel density estimator~\cite{alpaydin1986introduction} can be utilized, e.g.,
	\begin{equation}
	\Psi(\mathbf{x}_m) = \frac{1}{|X|(h_K)^{d^x}}\sum_{m'}K\left(\frac{\mathbf{x}_m-\mathbf{x}_{m'}}{h_K}\right),\label{eq:KDE}
	\end{equation}
	with the Gaussian kernel,
	$
	K(\mathbf{d}) = \left(\frac{1}{\sqrt{2\pi}}\right)^{d_x}\exp\left(-\frac{||\mathbf{d}||^2_2}{2}\right).
	$}
In order to make the approximation satisfy the local consistency condition, we need to take this distribution into account.
Then, $\Psi(\mathbf{x})$ serves as the weight of the importance sampling.
The Markov chain is approximated as
\begin{equation}
P_{nm} = \frac{\mathcal{N}(\mathbf{x}_m:\mu(h), \Sigma(h))/\Psi(\mathbf{x}_{m})}{\sum_{m'}\mathcal{N}(\mathbf{x}_{m'}:\mu(h), \Sigma(h))/\Psi(\mathbf{x}_{m'})}. \label{eq:approx_Gaussian}
\end{equation}
That is, the transition probabilities to samples in a higher density region are adjusted lower so that the \textit{local properties} of the Markov chain~\eqref{eq:approx_Gaussian} are \textit{consistent} with those of the original (uncontrolled) SDE~\eqref{eq:conti_dyn}.
Note that for each state, this only requires computing the $|X|$ dimensional vector and normalizing it.
In general, $\mu(h)$ and $\Sigma(h)$ are approximated as $\mu(h)\approx\mathbf{x}_n+h\mathbf{f}(\mathbf{x}_n)$ and $\Sigma(h)\approx h\sigma^2G(\mathbf{x}_n)G(\mathbf{x}_n)^T$.
However, this might cause some problems when $\Sigma(h)$ is a singular matrix; $h\sigma^2G(\mathbf{x}_n)G(\mathbf{x}_n)^T$ is singular whenever $d_u<d_x$.
We can approximate $\mu(h)$ and $\Sigma(h)$ by integrating the moment dynamics of the linearized SDE for $t\in[0,h]$:
\begin{align}
&\dot{\mu}(t) = A\mu(t)+\mathbf{c},\label{eq:mean_dyn}\\
&\dot{\Sigma}(t) = A\Sigma(t)+\Sigma(t)A^T+BB^T, \label{eq:Cov_dyn}\\
& \mu(0)=\mathbf{x}_n,~\Sigma(0) = 0,
\end{align}
where $A = \left.\frac{d\mathbf{f}}{d\mathbf{x}}\right|_{\mathbf{x}=\mathbf{x}_n}$, $B=\sigma G(\mathbf{x}_n)$ and $\mathbf{c}=\mathbf{f}(\mathbf{x}_n)-A\mathbf{x}_n$.
Note that if a linear system $(A,B)$ is controllable, which is the case for many systems, the solution of the Lyapunov equation (\ref{eq:Cov_dyn}), $\Sigma(t)$, is nonsingular for all $t>0$.
The transition probability of this Markov chain converges to that of (\ref{eq:conti_dyn}) as $|X|\rightarrow\infty$ and $h\rightarrow0$ \cite{kushner2013numerical}.
Also, one can utilize a local Gaussian distribution which truncates the tails of the distribution to make $P$ sparse.

\subsection{Multi-Scale Abstraction on Graph: Diffusion Wavelets}
From now on, we consider $T=P^T$ for notation simplicity.
Then $T_{nm}$ represents a transition probability from $\mathbf{x}_m$ to $\mathbf{x}_n$.
The Markov chain, $T$, obtained by discretizing the diffusion process ((\ref{eq:conti_dyn}) with $\mathbf{u}=0$) is known to have some interesting properties, namely: \textit{local}, \textit{smoothing} and \textit{contractive} \cite{coifman2006diffusion}.
From any initial point, $\delta_m$, the agent (numerically) transitions to only a few of its neighbors (i.e., $T\delta_m$ has a small support) and $T^j\delta_m$ is a smooth probability distribution.
Also since $||T||_2\leq1$, a dimension of subspace, $V_j$, which is $\epsilon$-spanned by $\{T^j\delta_m\}_{m\in X}$\footnote{With a slight abuse of notation, we will use $m\in X$ to denote $\mathbf{x}_m\in X$.} monotonically decreases as $j$ increases and $V_0\supseteq V_1\supseteq\cdots\supseteq V_j\supseteq\cdots$; especially for an irreducible Markov chain, $\text{dim}(V_j)\rightarrow 1$ as $j$ increases and a limit of $V_j$ corresponds to the stationary distribution of the Markov chain.

Let $W_j$ be an orthogonal complement of $V_{j+1}$ into $V_j$, i.e., $V_j = V_{j+1}\oplus W_j$ and suppose the bases $\Phi_j$ and $\Psi_j$ span $V_j$ and $W_j$, respectively.
By using the aforementioned properties of $T$, \textit{diffusion wavelets} constructs a hierarchical structure of a set of well-localized bases $\Phi_j$ and $\Psi_j$ called \textit{scaling} and \textit{wavelet functions}, respectively.
In order to explain the procedure, we introduce some notations used in \cite{coifman2006diffusion} with slight changes.
Let $[L]^{\Phi_{j+1}}_{\Phi_j}$ be the matrix representing the operator $L$ w.r.t. the basis $\Phi_j$ in the domain and $\Phi_{j+1}$ in the range, and $[B]_{\Phi_j}$ be a set of vectors $B$ represented on a basis $\Phi_j$, where the columns of $[B]_{\Phi_j}$ are the coordinates of the vectors $B$ in the coordinates $\Phi_j$.

The diffusion wavelet tree starts being constructed with a fixed precision $\epsilon>0$, the basis $\Phi_0=\{\delta_m\}_{m\in X}$ and the operator $T_0:=[T]^{\Phi_0}_{\Phi_0}=T$ on the basis $\Phi_0$.
Consider the set of functions $\tilde{\Phi}_0=\{T\delta_m\}_{m\in X}$, which is a set of columns of $[T]^{\Phi_0}_{\Phi_0}$.
Because $T$ is local, these functions are well-localized.
The algorithm gets a basis $\Phi_1=\{\phi_{1,m}\}_{m\in X_1}$ ($X_1$ is defined as this index set) written on the basis $\Phi_0$ by carefully orthogonalizing $\tilde{\Phi}_0$ to preserve being well-localized and to span a subspace which is $\epsilon$-close to the span$(\tilde{\Phi}_0)$.
Note that the elements of $\Phi_1$ are coarser than the elements of $\Phi_0$ since $T$ is smoothing and $|X_1|\leq|X|$ since $T$ is contractive.
The algorithm stores the information of the new basis in $[\Phi_1]_{\Phi_0}$ and computes the \textit{compressed} operator $T_1=[T^2]^{\Phi_1}_{\Phi_1}=[\Phi_{1}]_{\Phi_0}^T[T^2]^{\Phi_0}_{\Phi_0}[\Phi_{1}]_{\Phi_0}$ w.r.t. $\Phi_1$ in the domain and the range.
Also, the wavelets $[\Psi_0]_{\Phi_0}$ are similarly obtained by looking at the columns of $I_{\langle\Phi_0\rangle}-[\Phi_1]_{\Phi_0}[\Phi_1]_{\Phi_0}^T$.
The algorithm proceeds in the same fashion to get $[\Phi_{j+1}]_{\Phi_j}$ and $T_{j+1}$ for $j=1,...,J-1$.
A pseudo-code of the algorithm is shown in Algorithm \ref{alg:DWT}.

\begin{algorithm}
	\caption{Pseudo-code for Diffusion Wavelet Tree}\label{alg:DWT}
	\begin{algorithmic}[1]
		\For {$j=0,1,...,J-1$}
		\State $[\Phi_{j+1}]_{\Phi_j}\gets \textsc{SparseQR}(T_j,\epsilon)$
		\State $T_{j+1} = [T^{2^{j+1}}]^{\Phi_{j+1}}_{\Phi_{j+1}}\gets[\Phi_{j+1}]_{\Phi_j}^TT_j^2[\Phi_{j+1}]_{\Phi_j}$
		\State $[\Psi_j]_{\Phi_j}\gets \textsc{SparseQR}(I_{\langle\Phi_j\rangle}-[\Phi_{j+1}]_{\Phi_j}[\Phi_{j+1}]_{\Phi_j}^T,\epsilon)$
		\EndFor
		\State \Return $\textsc{DWT}\gets\{[\Phi_{j+1}]_{\Phi_j}, [\Psi_j]_{\Phi_j}; \forall j=0,...,J-1\}$
	\end{algorithmic}
\end{algorithm}

A set of basis functions at level $j$ can be written in the original coordinate (or can be \textit{unpacked}) as:
\begin{align}
\Phi_j&=[\Phi_j]_{\Phi_0}\nonumber\\
&=[\Phi_{j-1}]_{\Phi_0}[\Phi_{j}]_{\Phi_{j-1}}\nonumber\\
&=[\Phi_1]_{\Phi_0}\cdots[\Phi_{j-1}]_{\Phi_{j-2}}[\Phi_{j}]_{\Phi_{j-1}},
\end{align}
which is represented as a $|X|\times|X_j|$ matrix.
Note that each column of $[\Phi_j]_{\Phi_0}$ can be viewed as an ``abstract-state" of the original Markov chain.
The subspace with basis $[\Phi_j]_{\Phi_0}$ is $j\epsilon$-close to the subspace spanned by $\{T^{1+2+2^2+\cdots+2^{j-1}}\delta_m=T^{2^j-1}\delta_m\}_{m\in X}$;
that is, at the scale $j$, where $2^j$ steps of the scale $0$ (original scale) are considered as one-step, there are only $|X_j|$ meaningful combinations of states and each combination, $[\Phi_j]_{\Phi_0}$, represents ``abstract-state''.

\section{Multiscale Global and Local Planning}
\subsection{Linearly-solvable MDP and linear Bellman equation}
With a set of discrete states, $X$, the state-space as well as time-axis discretized version of SOC is formulated as the Markov decision process (MDP).
Equations representing the problem have the same form as (\ref{eq:discrete_dyn_pas})-(\ref{eq:discrete_cost}).
Because the cost is average over an infinite horizon, the optimal cost-to-go value,
\begin{equation}
c := \min_\pi J^\pi(\mathbf{x}),
\end{equation}
does not depend on the initial state, which is problematic since the optimal policy is reconstructed from the ``difference'' of the cost-to-go between states.
Consider the optimal cost-to-go function for the finite horizon MDP:
\begin{equation}
v_K(\mathbf{x}) := \min_\pi E\left[\sum^{K}_{k=0} hq(\mathbf{x})+D_{KL}(\pi(\cdot|\mathbf{x})||p(\cdot|\mathbf{x}))\right].
\end{equation}
Using $v_K$ for $K\rightarrow\infty$, the differential cost-to-go function is defined as:
\begin{equation}
v(\mathbf{x}) := v_k(\mathbf{x})-Kc.
\end{equation}
Then $c$ and $v$ satisfies the Bellman equation:
\begin{align}
&hc+v(\mathbf{x}) \nonumber \\
&= \min_\pi(hq(\mathbf{x})+D_{KL}(\pi(\cdot|\mathbf{x})||p(\cdot|\mathbf{x}))+E_{\mathbf{x}'\sim\pi(\cdot|\mathbf{x})}[v(\mathbf{x}')]). \label{eq:Bellman}
\end{align}
By defining the (differential) desirability function, $$z(\mathbf{x}) = \exp(-v(\mathbf{x})),$$ and the linear operator $\mathcal{G}[z](\mathbf{x})=\sum_{\mathbf{x}'}p(\mathbf{x}'|\mathbf{x})z(\mathbf{x}')$, we can rewrite the right side of the Bellman equation as: {\small
\begin{align}
&\min_\pi\left(hq(\mathbf{x})+E_{\mathbf{x}'\sim\pi(\cdot|\mathbf{x})}\left[\log\left(\frac{\pi(\mathbf{x}'|\mathbf{x})}{p(\mathbf{x}'|\mathbf{x})\exp(-v(\mathbf{x}'))}\right)\right]\right)= \nonumber\\
&\min_\pi\left(hq(\mathbf{x})-\log\mathcal{G}[z](\mathbf{x})+D_{KL}\left(\pi(\mathbf{x}'|\mathbf{x})||\frac{p(\mathbf{x}'|\mathbf{x})z(\mathbf{x}')}{\mathcal{G}[z](\mathbf{x})}\right)\right). \label{eq:Bellman_proc}
\end{align}}
Note that $\pi$ only affects the $D_{KL}$ term.
Then, the optimal policy is obtained analytically:
\begin{equation}
\pi^*(\mathbf{x}'|\mathbf{x}) = \frac{p(\mathbf{x}'|\mathbf{x})z(\mathbf{x}')}{\mathcal{G}[z](\mathbf{x})}. \label{eq:opt_pi}
\end{equation}
Substituting \eqref{eq:opt_pi} into \eqref{eq:Bellman_proc} yields the linear Bellman equation as:
\begin{equation}
\exp(-hc)z(\mathbf{x})=\exp(-hq(\mathbf{x}))\mathcal{G}[z](\mathbf{x}). \label{eq:lin_Bellman}
\end{equation}
Let $\mathbf{z}$ be the $|X|$-dimensional column vector of associated desirability function.
Then the linear Bellman equation (\ref{eq:lin_Bellman}) can be expressed in a matrix form
\begin{equation}
\lambda\mathbf{z} = QP\mathbf{z}, \label{eq:lin_Bellman2}
\end{equation}
where $\lambda = \exp(-hc)$ is the largest eigenvalue and $Q = \text{diag}(\exp(-hq(\mathbf{x})))$ \cite{todorov2009efficient}.
Note that (\ref{eq:lin_Bellman2}) is an eigenvalue problem with a $|X|\times |X|$ matrix.
There are various methods to solve an eigenvalue problem; one general way is the power iteration method.
An interesting point is that the power iteration method for this problem is an exact counterpart of the value iteration for standard MDPs but the convergence of the power iteration is faster \cite{todorov2009eigenfunction}.
Despite this advantage, however, the problem becomes intractable as the size of $X$ increases, and the efficient solution method is essential.

\subsection{(Global) Planning with Compressed MDPs}
Rather than solving the original $|X|\times |X|$ eigenvalue problem, we can treat the lower-dimensional coarsened problem.
Suppose a set of ``abstract-state" at level $j$, $\Phi_j$, is utilized as a set of bases for the original problem, which means the problem is viewed in a lower resolution with $2^j$ time scale.
Then, $\mathbf{z}$ is approximated as a linear combination of this set:
\begin{equation}
\hat{\mathbf{z}}_j=\Phi_j\mathbf{w_j},
\end{equation}
and the original problem (\ref{eq:lin_Bellman2}) is also written in these bases:
\begin{equation}
\hat{\lambda}_j\Phi_j\mathbf{w_j} = QP\Phi_j\mathbf{w_j}. \label{eq:approx_eigen}
\end{equation}
Since the columns of $\Phi_j$ are orthogonal, multiplying both sides of \eqref{eq:approx_eigen} by $\Phi_j'$ yields a compressed problem:
\begin{equation}
\hat{\lambda}_j\mathbf{w_j} = M_j\mathbf{w_j}, \label{eq:lin_Bellman3}
\end{equation}
where $M_j:=\left[QP\right]^{\Phi_j}_{\Phi_j}=\Phi_j^TQP\Phi_j$ is the $|X_j|\times|X_j|$ matrix with $M_0 = QP$.
Note that the compressed problem (\ref{eq:lin_Bellman3}) is much more tractable than the original problem (\ref{eq:lin_Bellman2}) if $|X_j|<<|X|$.

\begin{algorithm}
	\caption{Pseudo-code for Global Planning}\label{alg:GP}
	\begin{algorithmic}[1]
		\Statex // Recursive compression of the operator
		\State $M_l \gets \Phi_{l}^TQP\Phi_{l}$
		\For {$j=l,...,J-1$} \Comment from fine to coarse
		\State $M_{j+1}\gets[\Phi_{j+1}]_{\Phi_j}^TM_j[\Phi_{j+1}]_{\Phi_j}$
		\EndFor
		\Statex // Recursive solution of MDP
		\State $\mathbf{w}_J\gets\textsc{PowerIteration}(M_J, \mathbf{1})$
		\For {$j=J-1,...,l$} \Comment from coarse to fine
		\State $\tilde{\mathbf{w}}_j\gets [\Phi_{j+1}]_{\Phi_j}w_{j+1}$
		\State $\mathbf{w}_j\gets\textsc{PowerIteration}(M_j, \tilde{\mathbf{w}}_j)$
		\EndFor
		\State \Return $\hat{\mathbf{z}}_l \gets \Phi_l\mathbf{w}_l$
	\end{algorithmic}
\end{algorithm}
The hierarchical structure of the diffusion wavelet tree can be utilized to solve the problem more efficiently.
Suppose that our goal is to obtain the $l$th level approximate solution.  
First, the matrix $M_{j+1}$ can be computed recursively as $M_{j+1}\gets[\Phi_{j+1}]_{\Phi_j}^TM_j[\Phi_{j+1}]_{\Phi_j}$ from the $l$th level.
Also, if an iterative method like a power iteration method is used to solve the eigenvalue problem, the solution of the $(j+1)$th level, $\mathbf{w}_{j+1}$, can provide a \textit{warm-start} point to the $j$th level problem;
that is, the iteration starts from $\tilde{\mathbf{w}}_j= [\Phi_{j+1}]_{\Phi_j}\mathbf{w}_{j+1}$ by unpacking the solution of the next level.
If $\hat{\mathbf{z}}_j$ is not sharply changed by scale $j$, this warm-start will significantly reduce the number of iterations.
In summary, in order to solve the global planning at scale $l$, the matrix $M_j$ is computed from the finer to the coarsest level (from $l$ to $J$) and the reduced problem is solved from the coarsest to the finer level (from $J$ to $l$), recursively.
The pseudo-code of this procedure is shown in Algorithm \ref{alg:GP}.

\subsection{(Local) Supplementary/Detailed Planning}
In the global planning phase, the solution of the original MDP is approximated using the bases at scale $l$.
The exact solution can be reconstructed by adding the wavelet functions in lower scales, $\Psi_{0:l-1}$, as supplementary bases since $V_0=V_l\oplus W_{l-1}\oplus W_{l-2}\oplus\cdots\oplus W_0$.
(Note that $\Phi_l$ and $\Psi_{0:l-1}$ contain $|X_l|$ and $|X|-|X_l|$ bases, respectively.)
It is obvious that when a larger number of bases is used, we will have a more exact solution, but the problem becomes intractable.
An appropriate subset of bases needs to be selected from $\Psi_{1:l-1}$.
Which bases in $\Psi_{0:l-1}$ are valuable to improve the quality of the solution?
The wavelet bases are built as being well-localized. Therefore, the above question can be restated as follows:
``Which regions of the domain should be treated more extensively?"

\begin{algorithm}
	\caption{Pseudo-code for Local Planning}\label{alg:LP}
	\begin{algorithmic}[1]
		\State $P^* \gets \text{diag}(P\hat{\mathbf{z}}_l)^{-1}P\text{diag}(\hat{\mathbf{z}}_l)$ 
		\State $\mathbf{p} \gets \left[\frac{\mathcal{N}(\mathbf{x}_m:\mu, \Sigma)}{\sum_{m'}\mathcal{N}(\mathbf{x}_{m'}:\mu, \Sigma)}\right]$ \Comment $\mathbf{p}$: $|X|$-dim row vector
		\State $\bar{\mathbf{p}}\gets \mathbf{p}[\Phi_l]_{\Phi_0}$ \Comment $\bar{\mathbf{p}}$: $|X_l|$-dim row vector
		\State $\bar{P}^* \gets [\Phi_l]_{\Phi_0}^T(P^*)^{2^l}[\Phi_l]_{\Phi_0}$ 
		\State $\bar{\mathbf{d}} \gets \bar{\mathbf{p}}/(k_{LP}/2^l)$ 
		\For {$t=2,...,(k_{LP}/2^l)$}
		\State $\bar{\mathbf{p}} \gets \bar{\mathbf{p}}\bar{P}^*$
		\State $\bar{\mathbf{d}} \gets \bar{\mathbf{d}}+\bar{\mathbf{p}}/(k_{LP}/2^l)$
		\EndFor
		\State $\mathbf{d} \gets \bar{\mathbf{d}}[\Phi_l]_{\Phi_0}^T$ \Comment $\mathbf{d}$: $|X|$-dim row vector
		\State $\mathbf{s} \gets \mathbf{d}\left|[\Psi_{0:(l-1)}]_{\Phi_0}\right|$ \Comment $\mathbf{s}$: $(|X|-|X_l|)$-dim row vector
		\State Choose subset of wavelet bases, $\Psi_{LP}$, from $\Psi_{0:(l-1)}$ based on the score, $\mathbf{s}$.
		\State $\mathbf{w}\gets\textsc{PowerIter}(\begin{bmatrix} M_l & \Phi_l^TM_0\Psi_{LP}\\ \Psi_{LP}^TM_0\Phi_l & \Psi_{LP}^TM_0\Psi_{LP} \end{bmatrix}, \begin{bmatrix} \mathbf{w}_l \\ \mathbf{0} \end{bmatrix})$
		\State \Return $\hat{\mathbf{z}}\gets\begin{bmatrix}\Phi_l & \Psi_{LP}	\end{bmatrix}\mathbf{w}$
	\end{algorithmic}
\end{algorithm}

In this work, we choose and utilize the supplementary bases which are localized in the region the agent is \textit{more likely} to visit during the $k_{LP}$-steps with the approximated optimal policy.
And, the supplementary bases will be discarded and adaptively re-chosen for every $k_{LP}$-steps.
Define a weighting of states as the occupancy measure:
\begin{equation}
d^{\pi^*}(\mathbf{x}) = \frac{1}{k_{LP}}\sum^{k_{LP}}_{k=1} Pr(\mathbf{x}[k]=\mathbf{x}|\mathbf{x}[0]=\mathbf{x}_{cur},\pi^*),~\forall \mathbf{x}\in X
\end{equation}
where $\mathbf{x}_{cur}\in\mathcal{X}$ is the current state of a robot. Of course, $\mathbf{x}_{cur}$ is not one of the discrete points, i.e., $\mathbf{x}_{cur}\notin\mathcal{X}$.
$Pr(\mathbf{x}[1]=\mathbf{x}|\mathbf{x}[0]=\mathbf{x}_{cur},\pi^*)$ is obtained as follows.
First, the passive transition probability, $p(\mathbf{y}|\mathbf{x}_{cur}),~\forall\mathbf{y}\in X$, is computed using the Gaussian approximation as in (\ref{eq:approx_Gaussian}).
The optimal policy (\ref{eq:opt_pi}) states that the transition probability to $\mathbf{x}'$ is proportional to the desirability of that state as well as the transition probability without any control.
In a similar way, we obtain the (approximate) transition probability as:
\begin{equation}
Pr(\mathbf{x}[1]=\mathbf{x}|\mathbf{x}[0]=\mathbf{x}_{cur},\pi^*)=\frac{p(\mathbf{x}|\mathbf{x}_{cur})z(\mathbf{x})}{\sum_{\mathbf{x}\in X} p(\mathbf{x}|\mathbf{x}_{cur})z(\mathbf{x})}. \label{eq:opt_tran1}
\end{equation}
The probability for the other $k$ can be simply computed using the optimal transition probability matrix $P^*=\text{diag}(P\mathbf{z})^{-1}P\text{diag}(\mathbf{z})$ as:
\begin{align}
&Pr(\mathbf{x}[k]=\mathbf{x}_m|\mathbf{x}[0]=\mathbf{x}_{cur},\pi^*) \nonumber\\
&=\sum_{n'}Pr(\mathbf{x}[k-1]=\mathbf{x}_{n'}|\mathbf{x}[0]=\mathbf{x}_{cur},\pi^*)P^*_{n'm}. \label{eq:opt_tran2}
\end{align}
Then, all wavelet functions are scored by how overlapped they are with $d^{\pi^*}(\mathbf{x})$ as:
\begin{equation}
\mathbf{s}=\mathbf{d}\left|[\Psi_{0:(l-1)}]_{\Phi_0}\right|,
\end{equation}
where $\mathbf{d}$ and $\mathbf{s}$ are $|X|$ and $(|X|-|X_l|)$-dimensional row vectors, whose components each correspond to the weighting of each state and the score of each basis, respectively.
Finally, the subset of wavelet bases, $\Psi_{LP}$, is selected from $\Psi_{0:(l-1)}$ based on the score $\mathbf{s}$, and $\mathbf{z}$ is then re-approximated using the new basis set.

The procedure computing $d^{\pi^*}(\mathbf{x})$ operating in the original discrete state space requires $|X|$-dimensional matrix-vector operations.
However, the policy obtained from $\hat{\mathbf{z}}_l$ is only valid between ``abstract-states'', $\{\phi_{l,k}\}_{k\in X_l}$, with the time scale $2^l$.
Using this insight, the procedure can be done in the compressed form with $\Phi_l$:
we compress the operator as $\bar{P}^*=[\Phi_l]_{\Phi_0}'(P^*)^{2^l}[\Phi_l]_{\Phi_0}$, compute $d^{\pi^*}(\mathbf{x})$ in the compressed domain and unpack the result (as shown in line 3 - 10 in Algorithm \ref{alg:LP}).
As a result, $|X|$-dimensional operations are reduced to $|X_l|$-dimensional operations.

\section{Receding Horizon Control in Continuous Time}
Using the desirability function on the set of sample states, the continuous-time optimal control sequence, $\mathbf{u}^*(t)$, needs to be computed.
We will compute the continuous optimal control sequence for the interval $\tau=h\times k_{RHC}$, but apply it for the smaller interval $\tau_r\leq\tau$ in a receding horizon control fashion.

The control can be computed by matching the 1st order moment of the original SDE (\ref{eq:conti_dyn}) and the optimal policy (\ref{eq:opt_pi}) for the MDP.
First, the optimal transition probability for $k_{RHC}$-steps, $Pr(\mathbf{x}[k_{RHC}]=\mathbf{y}|\mathbf{x}[0]=\mathbf{x}_{cur},\pi^*)$, is computed in the same way as \eqref{eq:opt_tran1}-\eqref{eq:opt_tran2}.
If the robot follows the optimal policy (\ref{eq:opt_pi}), the expected state after $\tau$ will be
\begin{equation}
\mathbf{y}_{new} = \sum_{\mathbf{y}\in X}\mathbf{y}Pr(\mathbf{x}[k_{RHC}]=\mathbf{y}|\mathbf{x}[0]=\mathbf{x}_{cur},\pi^*).
\end{equation}
Then, the control that matches the state mean of the robot after $\tau$ with $\mathbf{y}_{new}$ is computed as the following proposition.
\begin{proposition}
	Consider the linearized 1st order moment dynamics of (\ref{eq:conti_dyn}),
	\begin{equation}
	\dot{\mu}_c(t)=A\mu_c(t)+G\mathbf{u}(t)+\mathbf{c}, \label{eq:con_mean_dyn}
	\end{equation}
	where $A = \left.\frac{d\mathbf{f}}{d\mathbf{x}}\right|_{\mathbf{x}=\mathbf{x}_{cur}},~G=G(\mathbf{x}_{cur})$ and $\mathbf{c}=\mathbf{f}(\mathbf{x}_{cur})-A\mathbf{x}_{cur}$.
    Suppose that the initial state is given by $\mu_c(0)=\mathbf{x}_{cur}$ and the control sequence,
	\begin{equation}
	\mathbf{u}^*(t)=-\sigma^2G^Te^{A^T(\tau-t)}\Sigma(\tau)^{-1}(\mu(\tau)-\mathbf{y}_{new}), \label{eq:opt_con}
	\end{equation}
	is applied to (\ref{eq:con_mean_dyn}) for $t\in[0,\tau]$, where $\mu(\tau)$ and $\Sigma(\tau)$ are the solutions of (\ref{eq:mean_dyn}) and (\ref{eq:Cov_dyn}) with $\mu(0)=\mathbf{x}_{cur},~\Sigma(0)=0$, respectively.
	Then, $\mu_c(\tau)=\mathbf{y}_{new}$.
\end{proposition}
\begin{proof}
	The solution of (\ref{eq:con_mean_dyn}) is given by
	\begin{equation}
	\mu_c(\tau) = e^{A\tau}\mu_c(0)+\int^{\tau}_0e^{A(\tau-t)}(G\mathbf{u}+\mathbf{c})dt. \label{eq:mean_sol}
	\end{equation}
	Substituting (\ref{eq:opt_con}) into (\ref{eq:mean_sol}) yields
	\begin{align}
	\mu_c(\tau) &= e^{A\tau}\mu_c(0)+\int^{\tau}_0e^{A(\tau-t)}\mathbf{c}dt \nonumber\\
	&-\sigma^2\int^{\tau}_0e^{A(\tau-t)}GG^Te^{A^T(\tau-t)}dt\Sigma(\tau)^{-1}(\mu(\tau)-\mathbf{y}_{new}) \nonumber\\
	&=\mu(\tau)-\Sigma(\tau)\Sigma(\tau)^{-1}(\mu(\tau)-\mathbf{y}_{new}) \nonumber\\
	&=\mathbf{y}_{new}\label{eq:mean_sol2}
	\end{align}
	We utilize the analytic solutions of (\ref{eq:mean_dyn}) and (\ref{eq:Cov_dyn}):
	\begin{align}
	&\mu(\tau)= e^{A\tau}\mu(0)+\int^{\tau}_0e^{A(\tau-t)}\mathbf{c}dt, \nonumber\\
	&\Sigma(\tau)=\int^{\tau}_0e^{A(\tau-t)}BB^Te^{A^T(\tau-t)}dt,
	\end{align}
	and $BB^T=\sigma^2GG^T$.
\end{proof}
\begin{remark}
It is also worth noting that $\mathbf{u}^*(t)$ is the solution of a deterministic optimal control problem of a dynamic system (\ref{eq:con_mean_dyn}) with cost function
$$
J = \int_0^{\tau}q(\mathbf{x}_{cur})+\frac{1}{2\sigma^2}\mathbf{u}(t)^T\mathbf{u}(t)dt,
$$
subject to $\mu_c(0)=\mathbf{x}_{cur}$ and $\mu_c(\tau)=\mathbf{y}_{new}$, which is also called an affine-quadratic optimal control problem.
\end{remark}

Finally, we'd like to emphasize that the time interval $\tau$ should be determined to be small as the linearization effect on the mean dynamics (\ref{eq:con_mean_dyn}) is not too significant.

\section{Numerical Examples}
\subsection{Robot Navigation in Fractal Shape Environment}
\begin{figure}[tp]
	\centering
	\subfigure[]{
		\includegraphics*[width=.4\columnwidth, viewport =52 30 380 300]{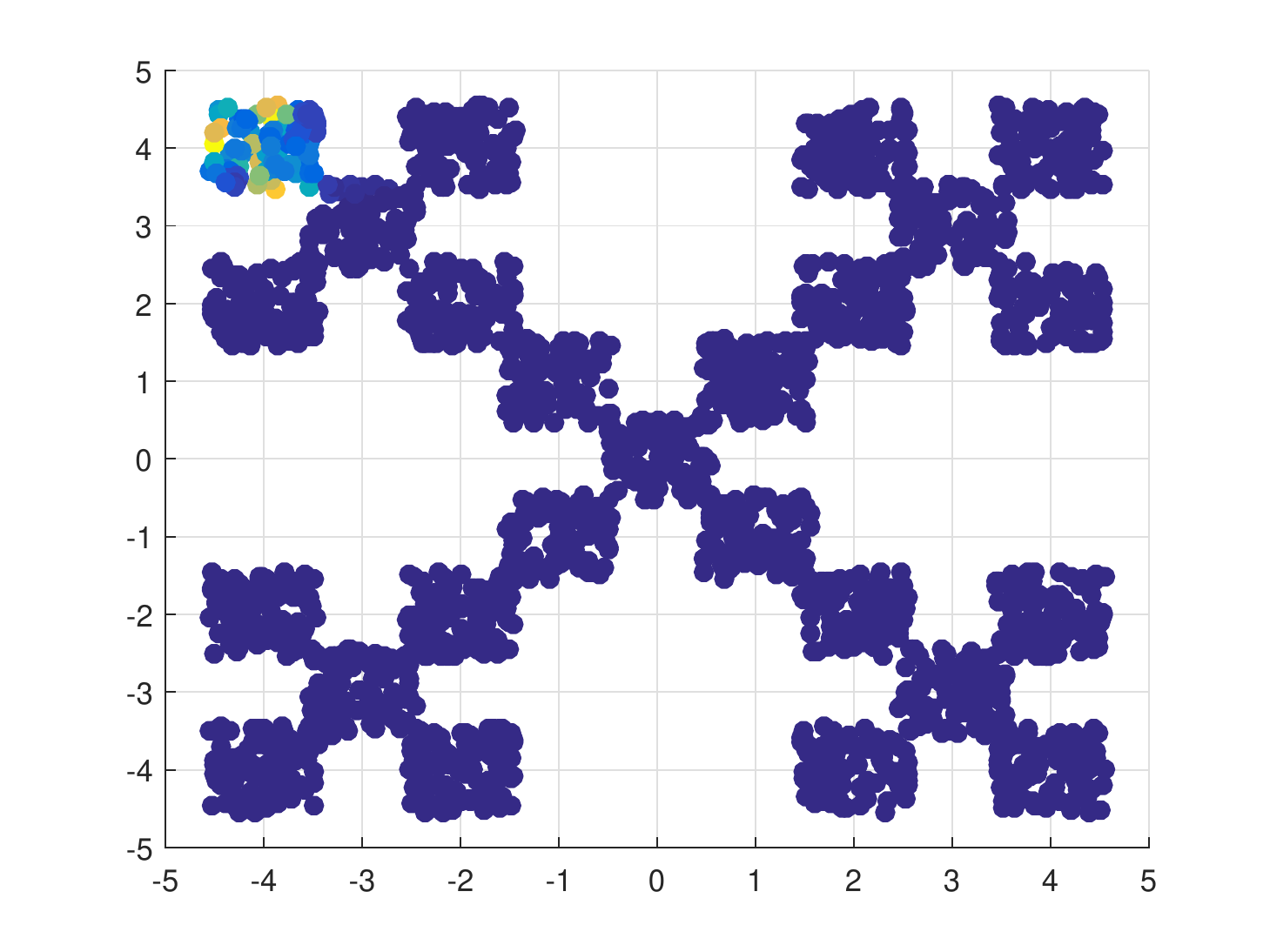}}
	\subfigure[]{
		\includegraphics*[width=.4\columnwidth, viewport =52 30 380 300]{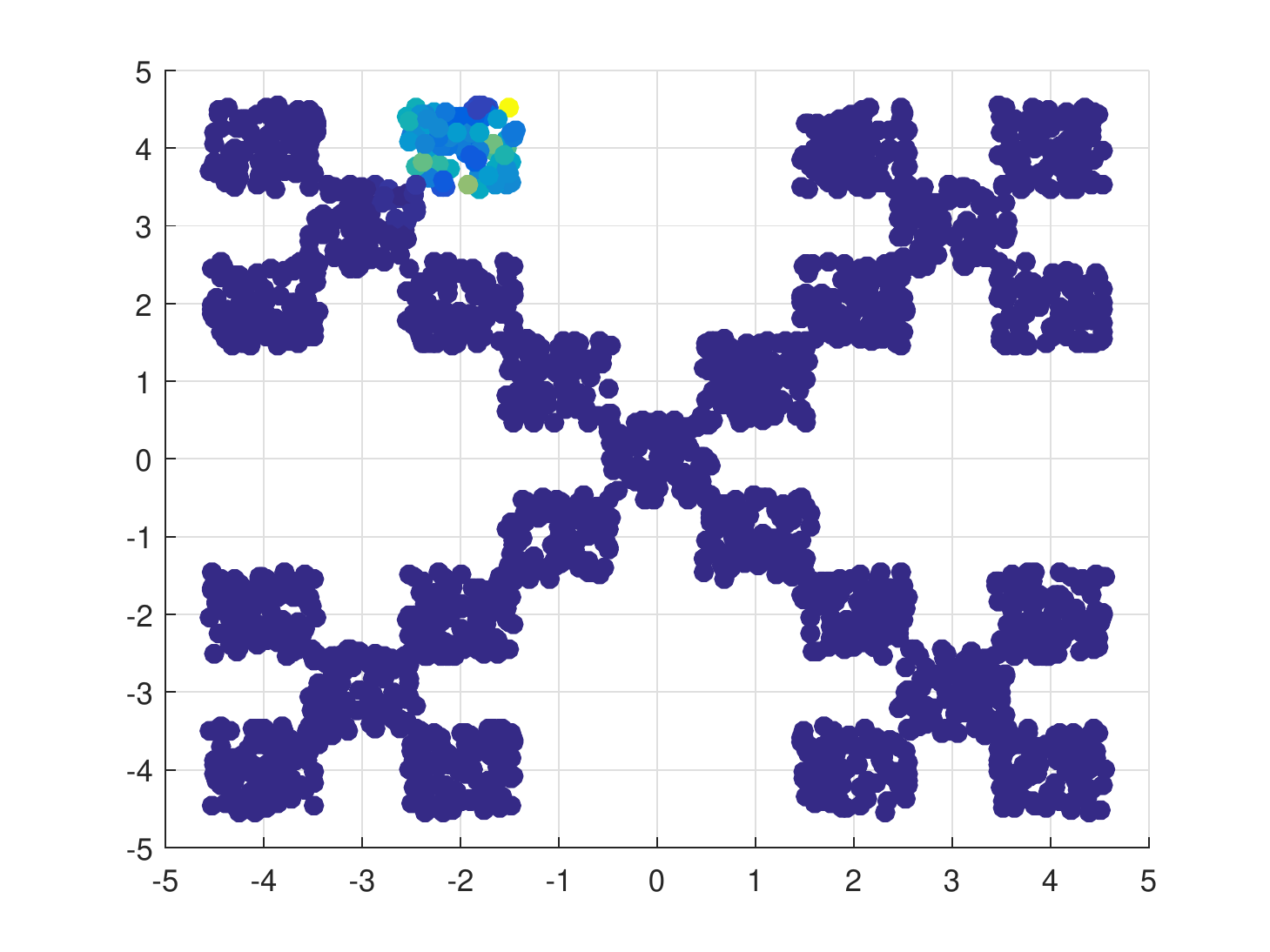}}
	\subfigure[]{
		\includegraphics*[width=.4\columnwidth, viewport =52 30 380 300]{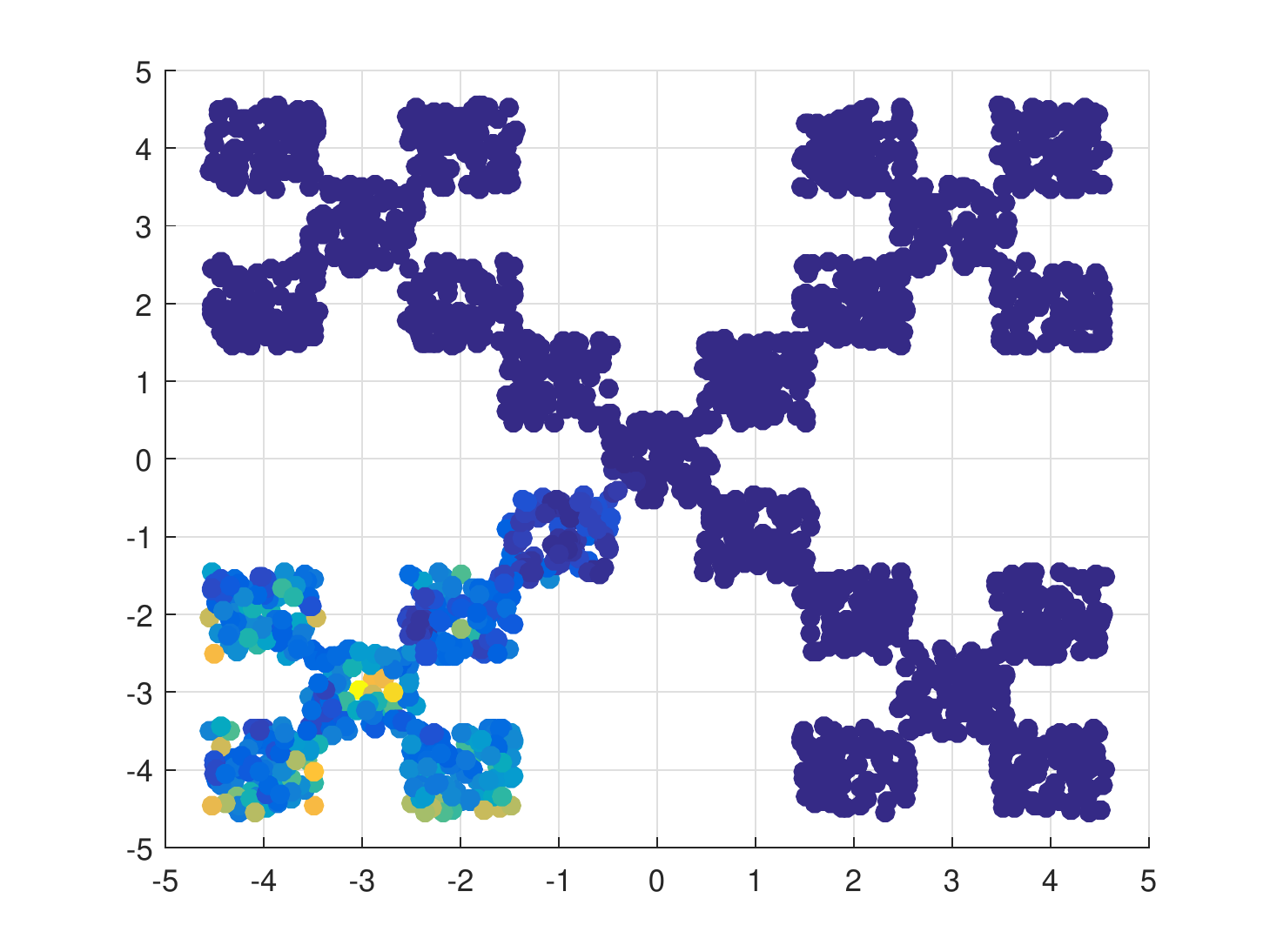}}
	\subfigure[]{
		\includegraphics*[width=.4\columnwidth, viewport =52 30 380 300]{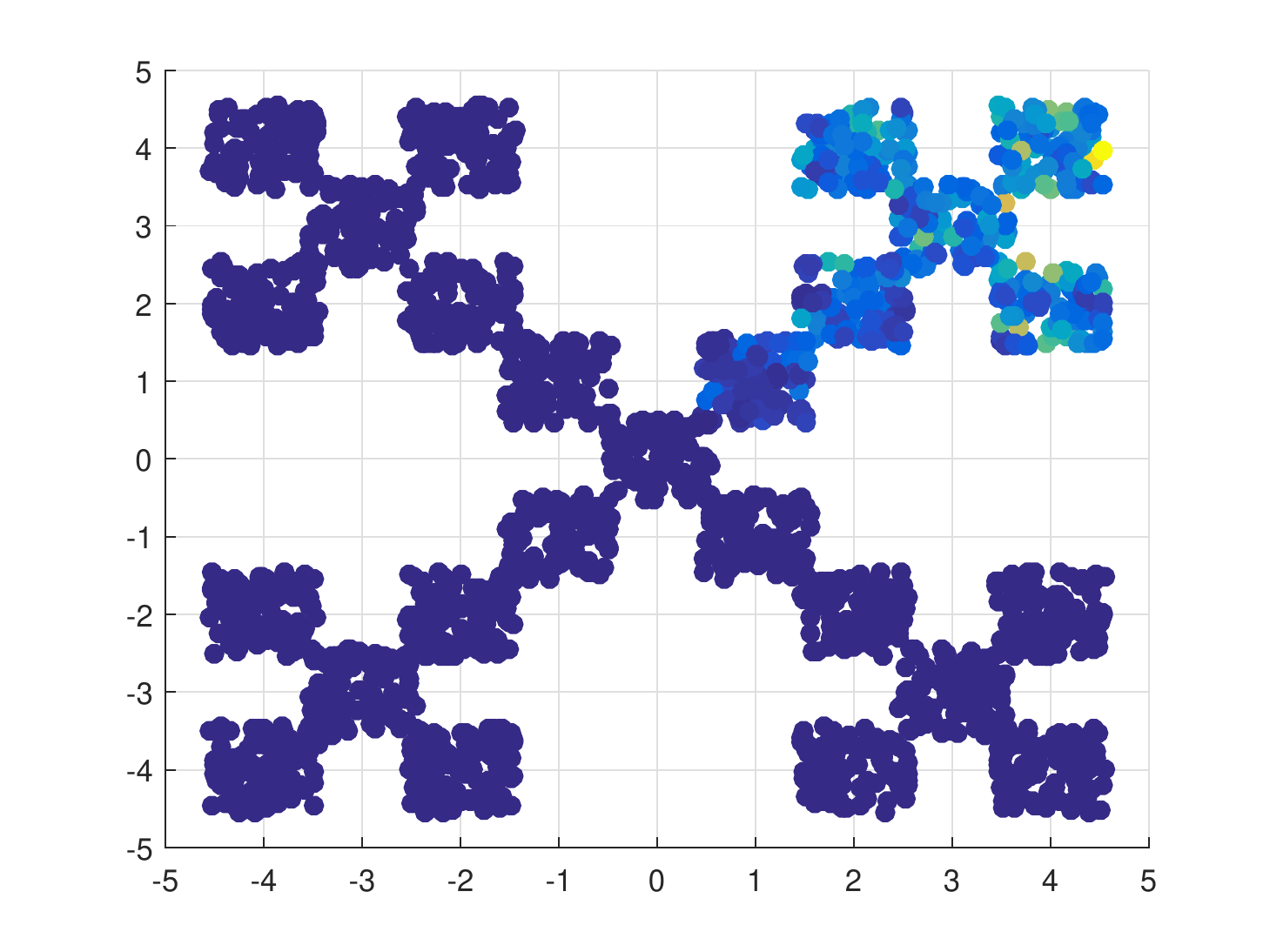}}
	\caption{(a)-(d) Some scaling functions at (a)-(b) level 8 and (c)-(d) level 12 in fractal environment example.}\vspace*{-.15in}
	\label{fig:scaling}
\end{figure}
\begin{figure}[t]
	\centering	
	\subfigure[cost function, $q(\mathbf{x})$]{\label{fig:fractal_cost}
		\includegraphics*[width=.4\columnwidth, viewport =48 30 380 300]{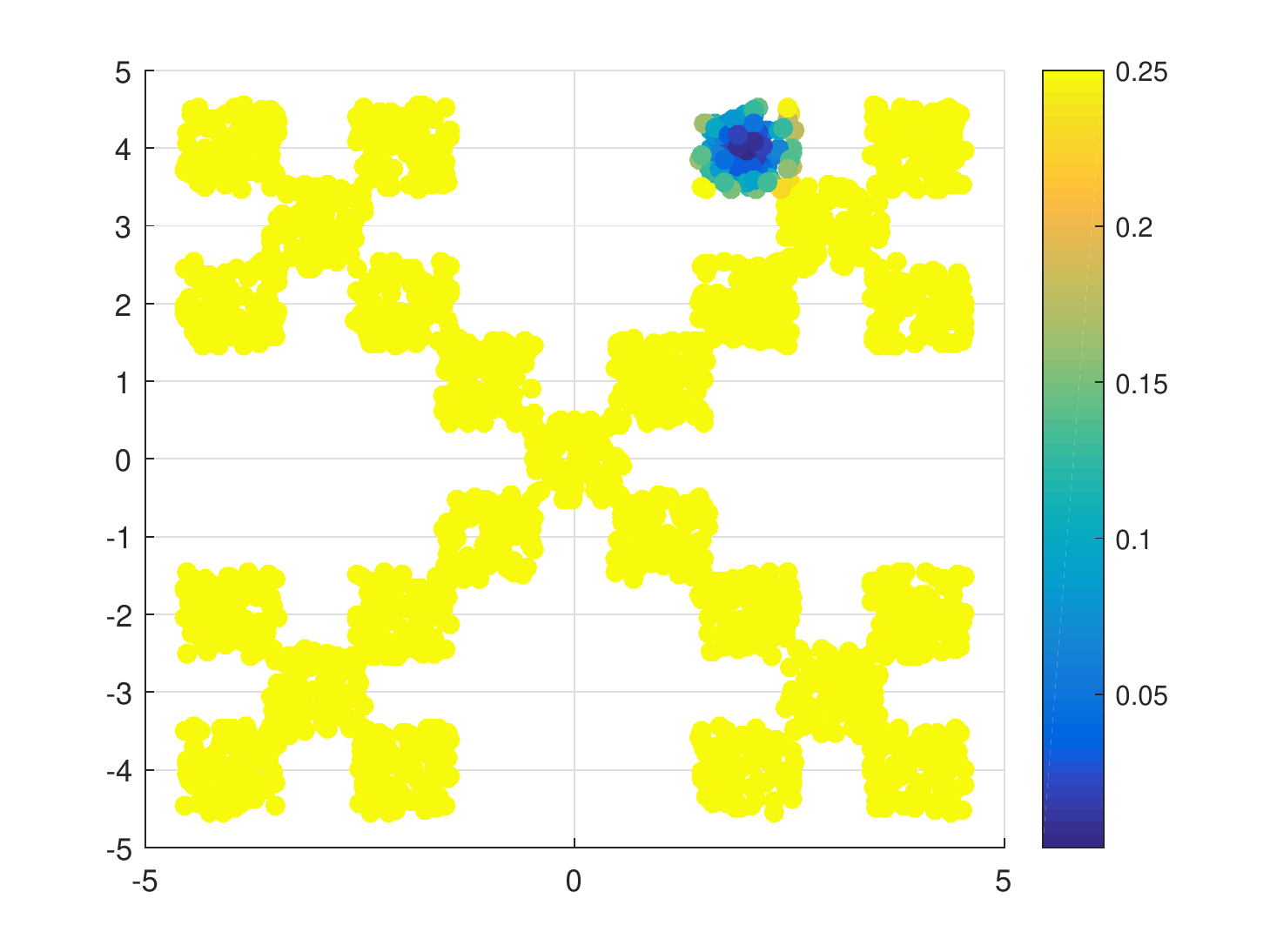}}
	\subfigure[value function, $v(\mathbf{x})$]{
		\includegraphics*[width=.4\columnwidth, viewport =48 30 380 300]{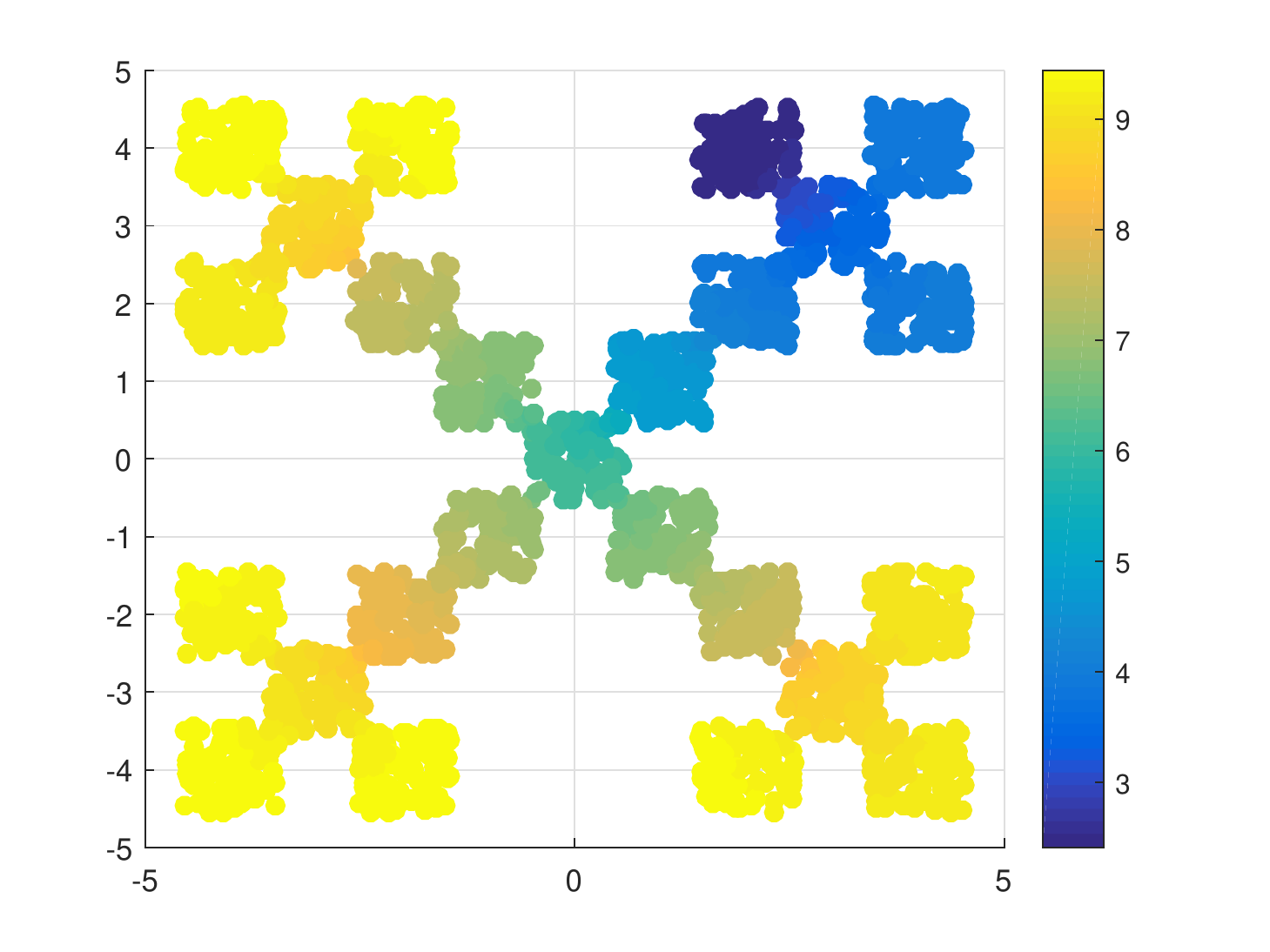}}
	\subfigure[$\hat{\mathbf{v}}_8=-\log(\Phi_8\mathbf{w}_8)$]{
		\includegraphics*[width=.4\columnwidth, viewport =48 30 380 300]{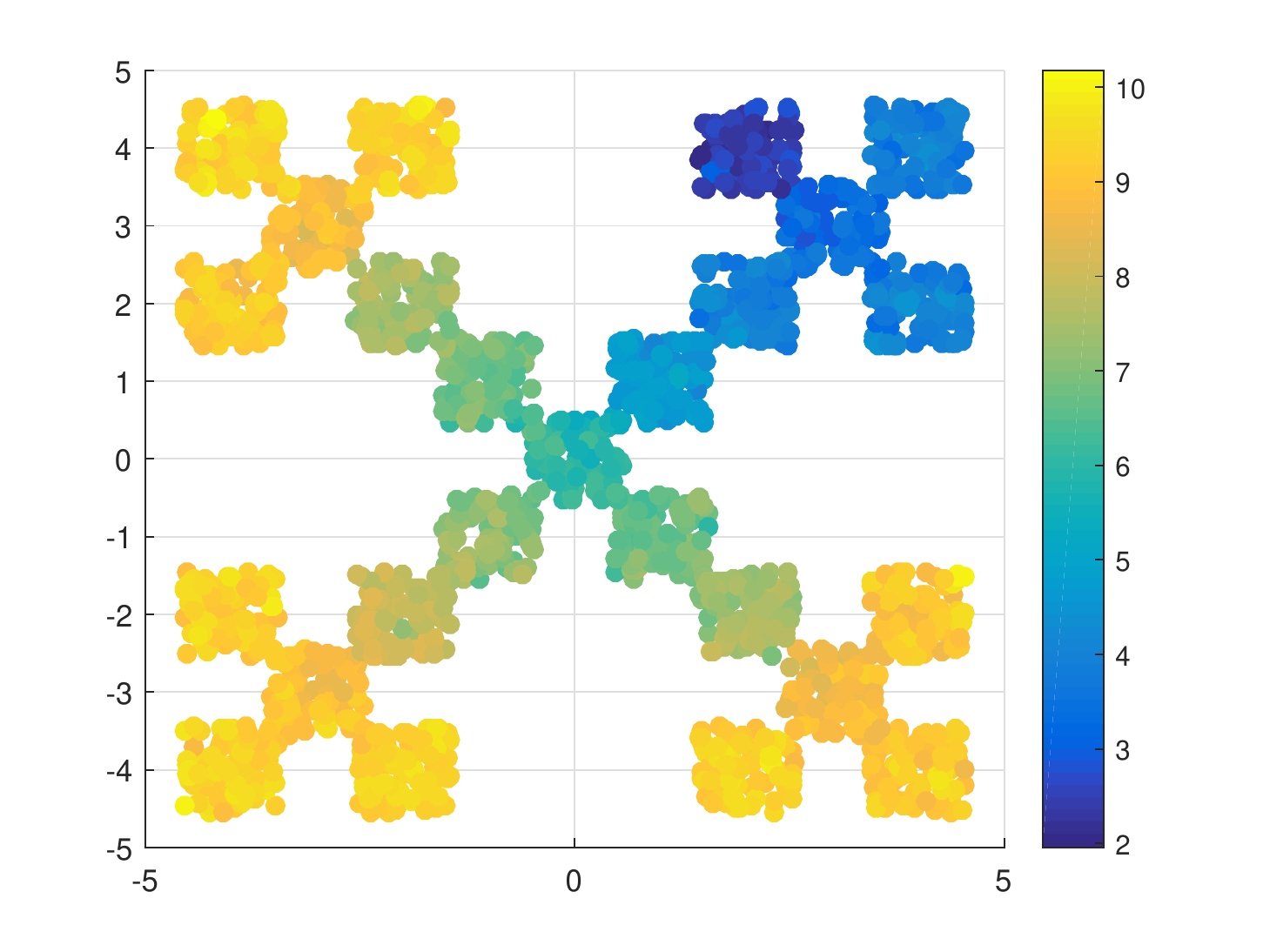}}
	\subfigure[$\hat{\mathbf{v}}_{12}=-\log(\Phi_{12}\mathbf{w}_{12})$]{
		\includegraphics*[width=.4\columnwidth, viewport =48 30 380 300]{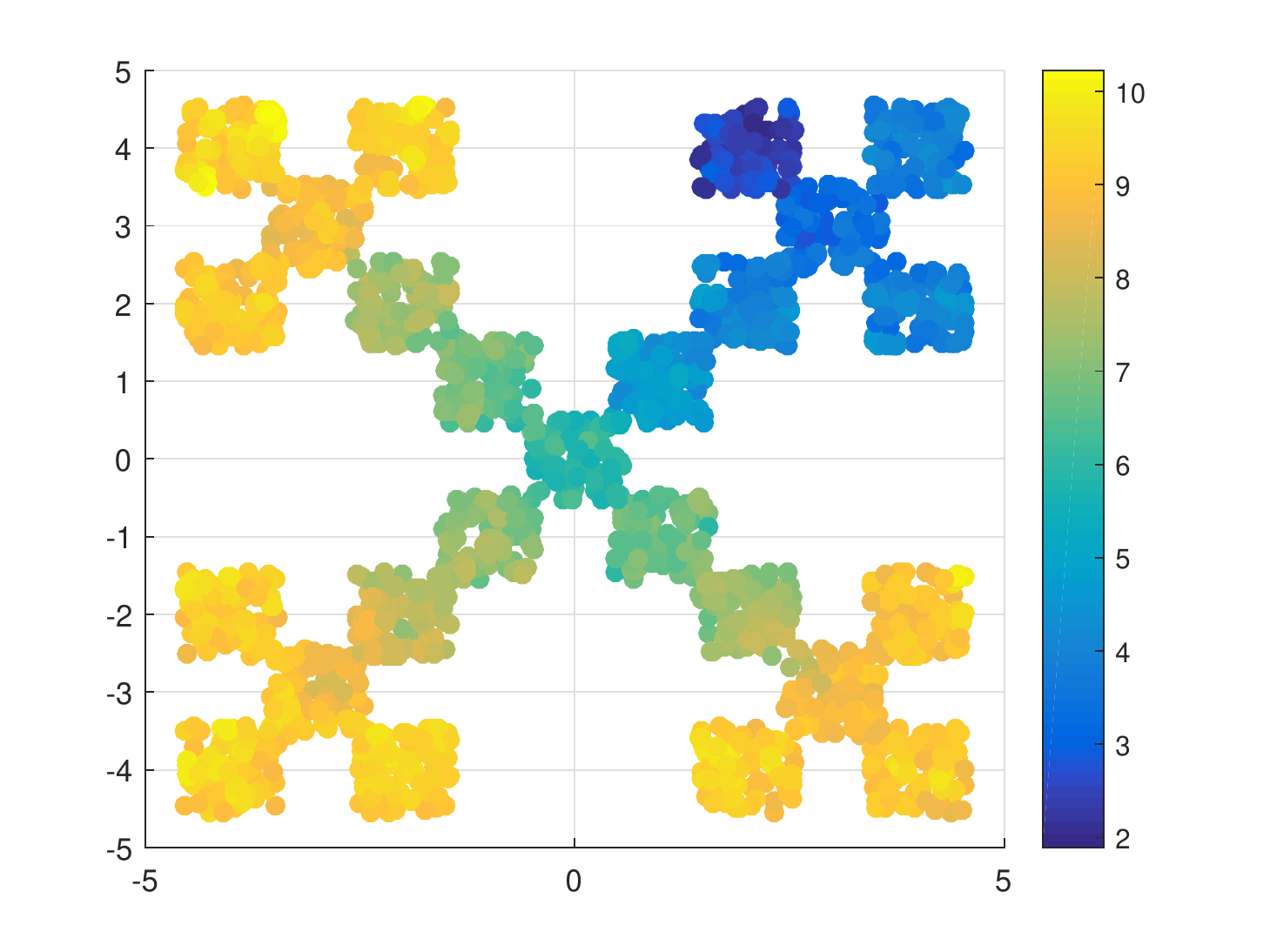}}
	\caption{Fractal environment example. Cost, value and approximated value function at level 8 and 12.}\vspace*{-.15in}
	\label{fig:cost_solution}
\end{figure}
\begin{figure}[t]
	\centering
	\subfigure{
		\includegraphics*[width=.8\columnwidth, viewport =10 0 415 162]{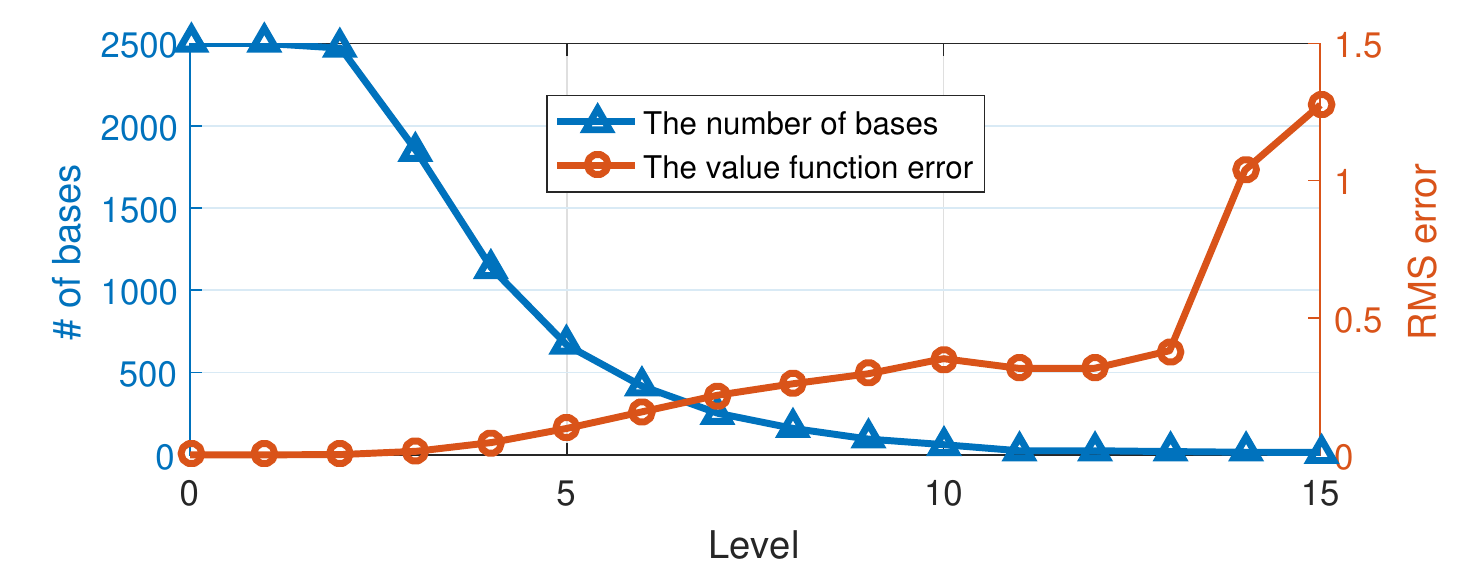}}\vspace*{-.05in}
	\subfigure{
		\includegraphics*[width=.8\columnwidth, viewport =10 0 415 162]{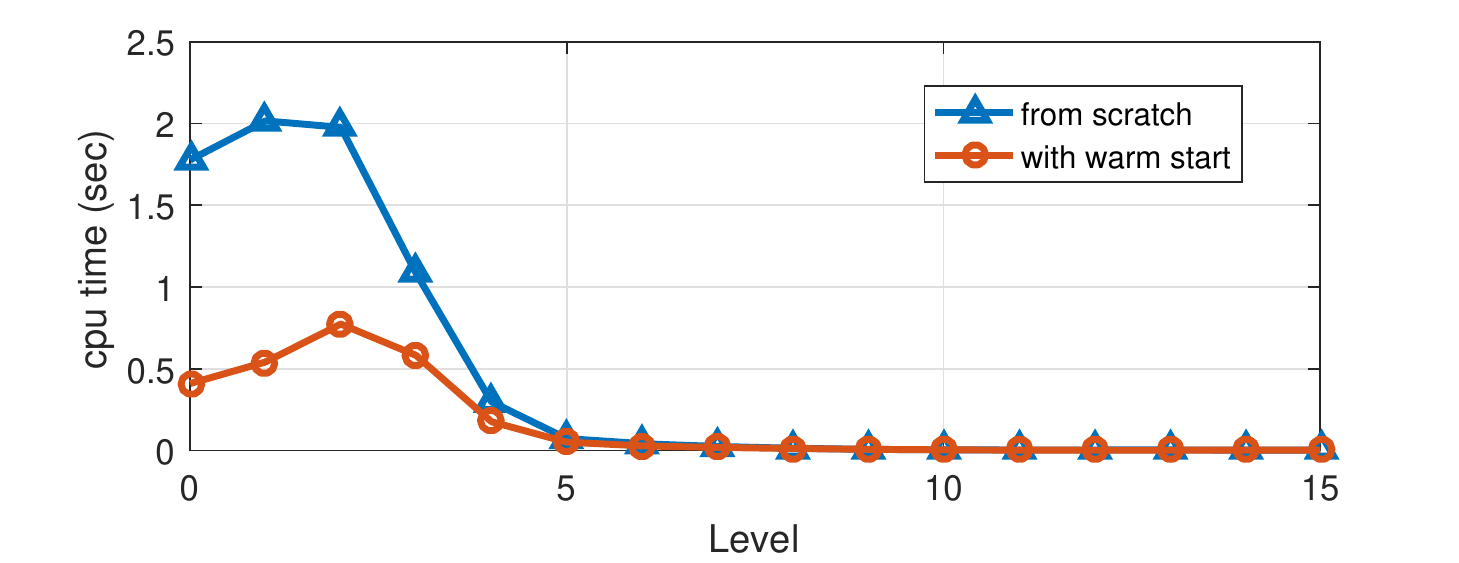}}\vspace*{-.05in}
	\caption{Fractal environment example. Results of Global planning. (upper) RMS errors are measured between $\mathbf{v}$ and $\hat{\mathbf{v}}$. (lower) Eigenvalue problems are solved using a MATLAB built-in function, \textit{eigs}.}\vspace*{-.15in}
	\label{fig:GP}
\end{figure}
\begin{figure}[t]
	\centering
	\subfigure[$err_{\hat{v}_8}$ before LP]{
		\includegraphics*[width=.45\columnwidth, viewport =48 30 380 295]{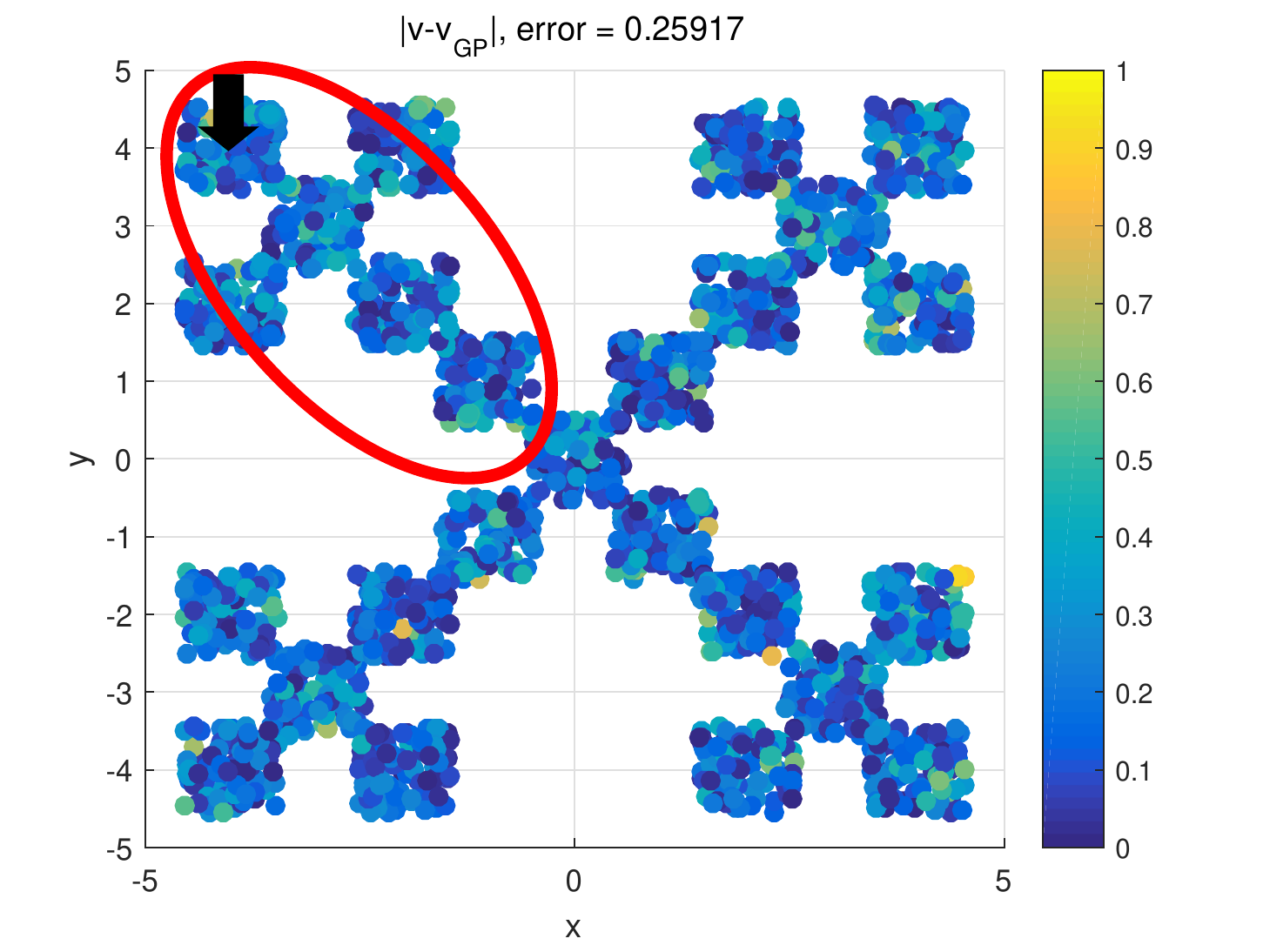}}
	\subfigure[$err_{\hat{v}_8}$ after LP]{
		\includegraphics*[width=.45\columnwidth, viewport =48 30 380 295]{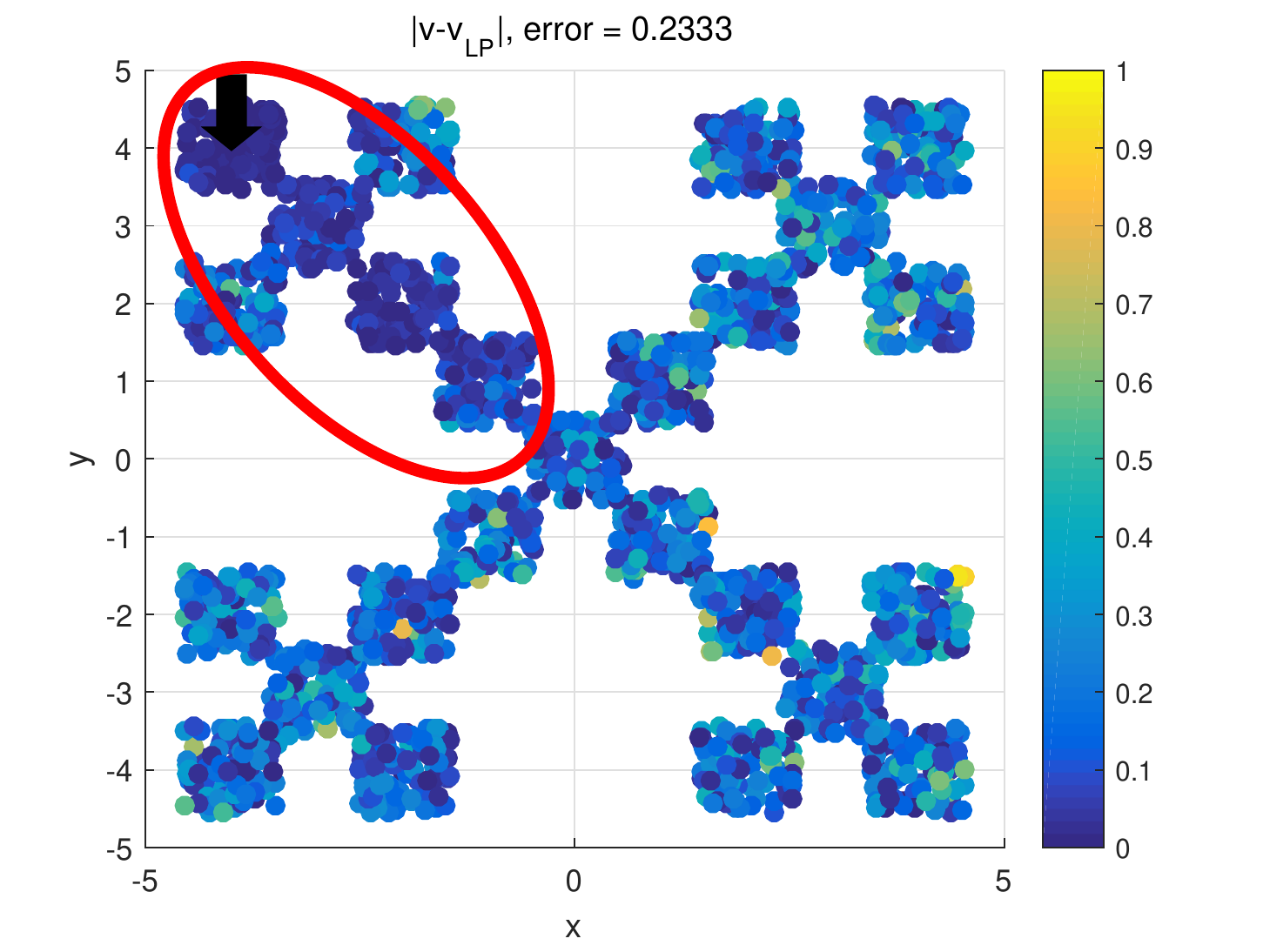}}
	\caption{Fractal environment example. Results of local planning. The errors of value functions are measured as $err_{\hat{v}_8}=|v(\mathbf{x})-\hat{v}_8(\mathbf{x})|$. The robot is assumed to be located in the center of the furthest left and uppermost room.}\vspace*{-.15in}
	\label{fig:LP}
\end{figure}
For the first example, we consider a simple two-dimensional stochastic single integrator in a fractal-like environment.
The environment consists of 5 group of rooms where one group is made up of 5 rooms as shown in Fig. \ref{fig:scaling}.
One can observe that the environment has 2 level self-similarity, which gives the problem a multiscale nature.
The dynamics is given by:
$$\mathbf{f}(\mathbf{x}) = \mathbf{0},~G(\mathbf{x}) = I_2$$
i.e., the position of a robot in configuration space, $\mathbf{x}\in \chi$, is controlled by the velocity input, $\mathbf{u}\in\mathbb{R}^2$.
We set $h=0.01$ and $\sigma = 1$.
In this setting, $\mu(h)=\mathbf{x}_{cur}$ and $\Sigma(h) = hI_2$ are obtained analytically, and the Markov chain is then constructed by a simple Gaussian probability with Euclidean distance, which has been extensively studied in the graph Laplacian-based solution approach for MDPs \cite{mahadevan2007proto,corneil2015attractor,chenmotion2016motion}.
Our stochastic optimal control problem formulation can be viewed as a generalization of them, since it can treat more general dynamics and cost functions.

In order to discretize the state space, 100 samples are obtained from each room, therefore there are 2500 discrete states in total.
Fig. \ref{fig:scaling} shows the abstraction results: there are some scaling functions in the diffusion wavelet tree at level 8 and 12.
It can be seen that at level 8 and 12, the scaling functions roughly represent each small room, and one group of 5 rooms, respectively.
The cost function of the problem is depicted in Fig. \ref{fig:fractal_cost}; it denotes the goal of the robot in this environment.
The optimal value function with the original bases and its approximations on the course bases are shown in Fig. \ref{fig:cost_solution}(b)-(d).
It can be observed from Fig. \ref{fig:GP}(a) that as the scale level increases, the number of bases are significantly reduced, but the value function is quite well-approximated.
We solved the eigenvalue problems \eqref{eq:lin_Bellman3} using a MATLAB built-in function, \textit{eigs}, with/without the warm start.
Fig. \ref{fig:GP}(b) shows that the approximate solution of the coarser level greatly helps the rapid convergence of the eigenvalue problem solver.
Finally, the result of local planning is shown in Fig. \ref{fig:LP}; the robot is assumed to be located in the center of the furthest left and uppermost room.
After the global planning, the approximated value function contains some error (see Fig. \ref{fig:LP}(a)), but the coarse level of the plan can be obtained (see Fig. \ref{fig:cost_solution}(c)-(d)).
Then it can evaluate which regions will be visited more likely, and make a detailed plan for those regions.
Fig. \ref{fig:LP}(b) shows that the value function error in the room located on the way to the goal (highlighted by the red ellipsoid) has been significantly reduced after the local planning.

\subsection{Quadrotor Control}
\begin{figure}[t]
	\centering
	\includegraphics*[width=1\columnwidth, viewport=70 170 900 470]{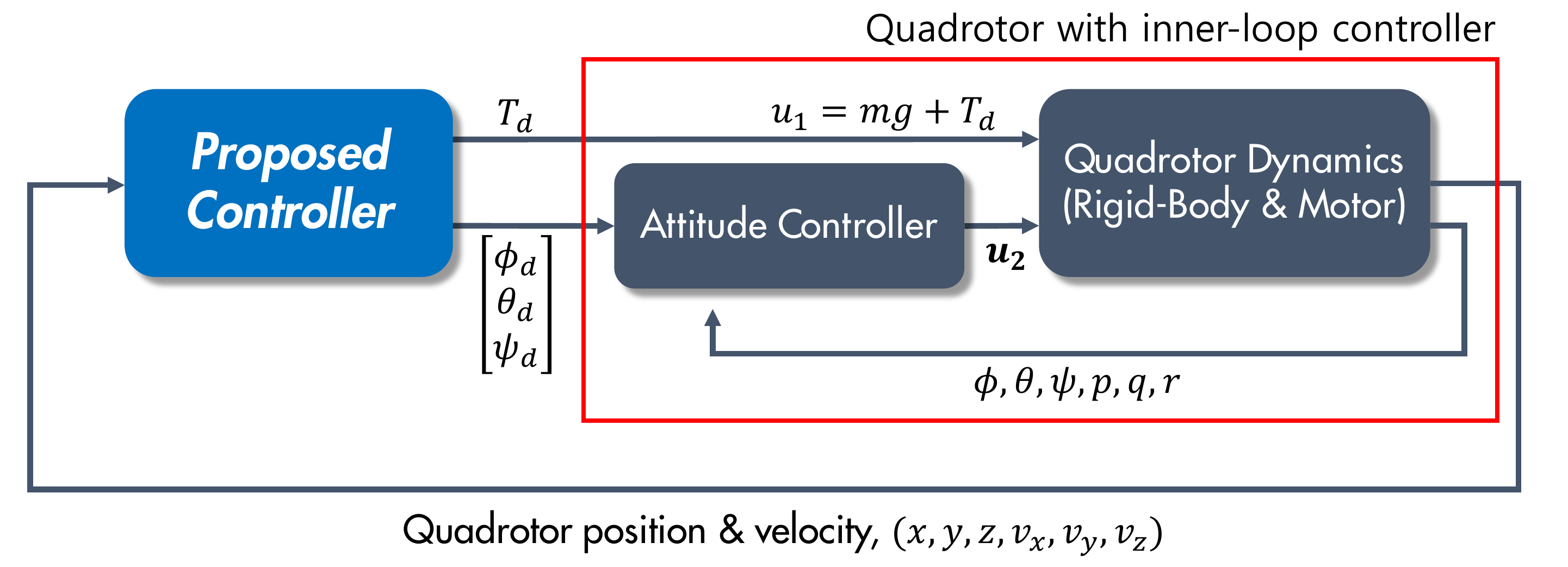}\vspace*{-.15in}
	\caption{Quadrotor control scheme.}
	\label{fig:quad_controller}\vspace*{-.15in}
\end{figure}
The second numerical experiment is about the control of a quadrotor.
We used a quadrotor dynamics introduced in \cite{michael2010grasp}.
The state of the quadrotor is given by the 3-dimensional position $\mathbf{r}=[x,y,z]^T$, velocity $\mathbf{v}=[v_x,v_y,v_z]^T$, orientation $[\phi, \theta, \psi]^T$ (which represent roll, pitch, and yaw angles, respectively) and angular velocity $[p,q,r]^T$.
The input is given by the linear combinations of forces from each rotor, $F_i$, as
\begin{align}
u_1 = \sum_{i=1}^{4}F_i,~\mathbf{u}_2 = L\begin{bmatrix}
0 & 1 & 0 & -1 \\ -1 & 0 & 1 & 0 \\ \mu & -\mu & \mu & -\mu
\end{bmatrix}\begin{bmatrix}
F_1\\F_2\\F_3\\F_4
\end{bmatrix},
\end{align}
where $L$ is the distance of the rotor axis from the center of the body and $\mu$ is a coefficient for the moment-force relation.
Then the 12-dimensional quadrotor dynamics is given by
\begin{align}
&\dot{\mathbf{r}} = \mathbf{v},~\dot{\mathbf{v}} = \begin{bmatrix}0\\0\\-g\end{bmatrix} + \begin{bmatrix}c\psi s\theta+c\theta s\phi s\psi\\s\psi s\theta-c\psi c\theta s\phi\\c\phi c\theta\end{bmatrix}u_1, \nonumber\\
&\begin{bmatrix}\dot{\phi}\\\dot{\theta}\\\dot{\psi}\end{bmatrix} = \begin{bmatrix}c\theta&0&-c\phi s\theta\\0&1&s\phi\\s\theta&0&c\phi c\theta\end{bmatrix}^{-1}\begin{bmatrix}p\\q\\r\end{bmatrix}, \nonumber\\
&\begin{bmatrix}\dot{p}\\\dot{q}\\\dot{r}\end{bmatrix} = -I^{-1}\begin{bmatrix}p\\q\\r\end{bmatrix}\times I\begin{bmatrix}p\\q\\r\end{bmatrix} + I^{-1}\mathbf{u}_2, \label{eq:dyn_quad_full}
\end{align}
where $g$ and $I$ denote the acceleration of gravity and the moment of inertia matrix, respectively; also, $c\cdot$ and $s\cdot$ are the cosine and sine functions.

Assume that the quadrotor embeds the PD-type altitude and attitude controller (shown as the red-box in Fig. \ref{fig:quad_controller}) as:
\begin{equation}
u_1 = mg+T_d,~\mathbf{u}_2 = I\begin{bmatrix}k_{p,\phi}(\phi_d - \phi) - k_{d,\phi}p \\ k_{p,\theta}(\theta_d - \theta) - k_{d,\theta}q \\ k_{p,\psi}(\psi_d - \psi) - k_{d,\psi}r\end{bmatrix},
\end{equation}
with $T_d = k_{p,z}(z_0 - z) - k_{d,z}v_z$, and its position $(x,y)$ is then controlled by the desired orientations~\cite{michael2010grasp}.
In order to obtain the reduced model, we linearized the quadrotor dynamics in the hovering state (with a fixed yaw angle, $\psi=\psi_d=0$) and considered that the linearization effect and the transient happened inside the red-box as noise.
Our control inputs for the reduced model are then the desired pitch $\theta_d$ and roll $\phi_d$ angles which are sent into the red box.
Also, the states are the position and velocity of the quadrotor (see Fig. \ref{fig:quad_controller}).
Then, the reduced dynamics are given as:
\begin{align}
d\mathbf{r}_{1:2} = \mathbf{v}_{1:2}dt,~d\mathbf{v}_{1:2} \approx \begin{bmatrix}g&0\\0&-g\end{bmatrix}\left(\begin{bmatrix}\theta_d\\\phi_d\end{bmatrix}dt+\begin{bmatrix}\sigma_xdw_x\\\sigma_ydw_y\end{bmatrix}\right).\nonumber
\end{align}

Since the dynamics can be decoupled with the subspaces for $(x,v_x)$ and $(y,v_y)$, respectively, we use 200 grid points as samples in each subspace.
The Markov chains are constructed independently and the diffusion wavelet trees of the scaling functions, $\bar{\Phi}_j$ and $\tilde{\Phi}_j$, are obtained for each subspace.
The $j$ level basis functions in the space for $(\mathbf{r}_{1:2}, \mathbf{v}_{1:2})$ are then computed by the \textit{Kroneker tensor products} of the scaling functions in each subspace, e.g., $\Phi_j = \bar{\Phi}_j\otimes\tilde{\Phi}_j$.
While level 0 has 40000 basis functions, which are all combinations of samples from both subspaces, at level 4 there are only 1444 basis functions.
Note that, once the tree is constructed, the same abstraction can be utilized for various tasks.
Fig. \ref{fig:quad} shows the resulting trajectories with the policies at the different levels and for two different cost functions.
To obtain the trajectories, we computed the optimal control sequence using \eqref{eq:opt_con} and passed it to the inner-loop controller of the quadrotor in the way of the receding horizon control.
At the higher level, because the basis functions are coarser, so are the resulting policies; it is observed from Fig. \ref{fig:quad} (b), (d) that the resulting policies at level 4 made detours.
It can be seen that, however, the resulting policies perform well even with far fewer basis functions.
\begin{figure}[t]
	\centering
	\subfigure[Task 1 with $0$ level policy]{
		\includegraphics*[width=.475\columnwidth, viewport =52 20 370 280]{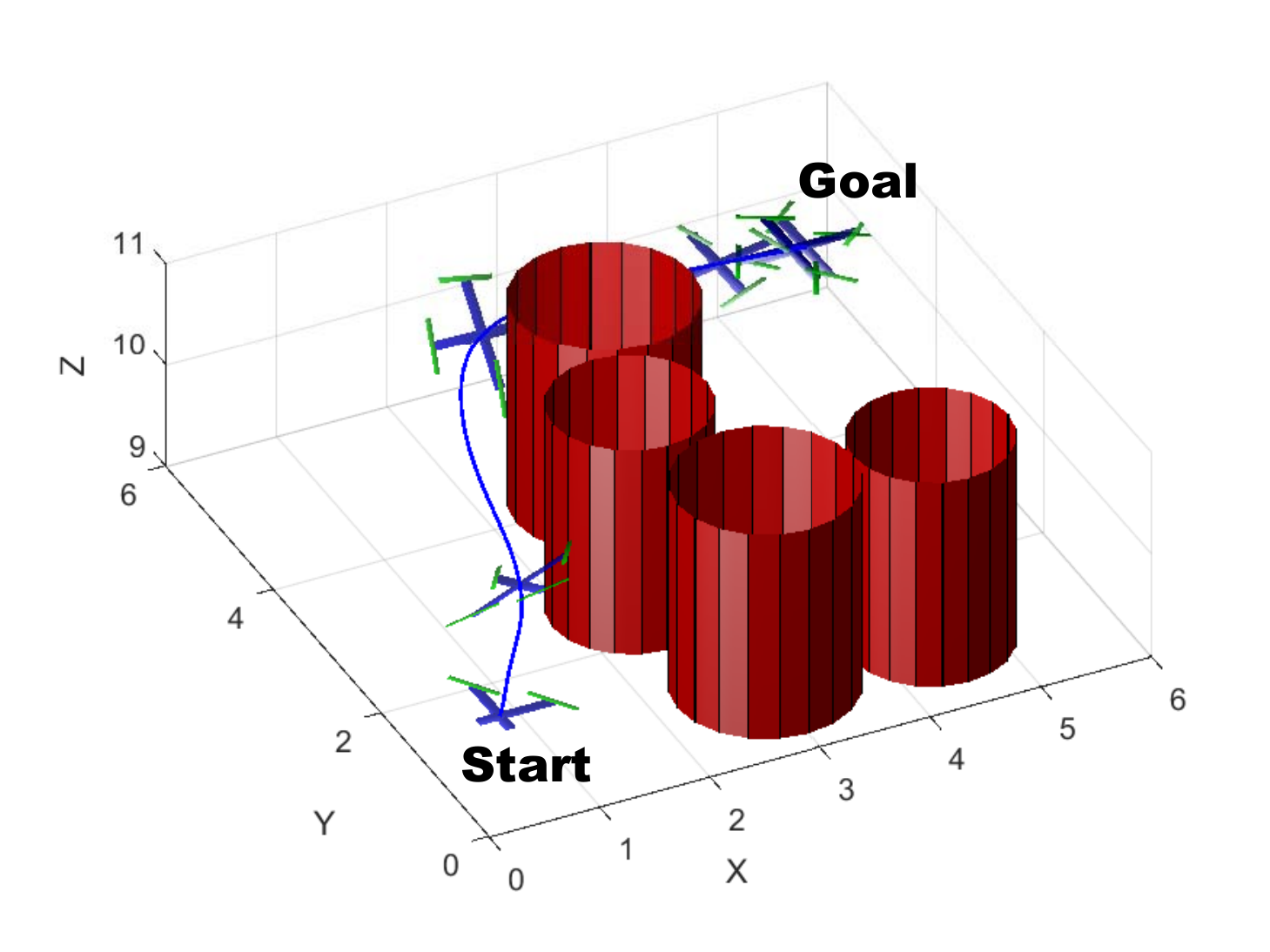}}
	\subfigure[Task 1 with  $4$ level policy]{
		\includegraphics*[width=.475\columnwidth, viewport =52 20 370 280]{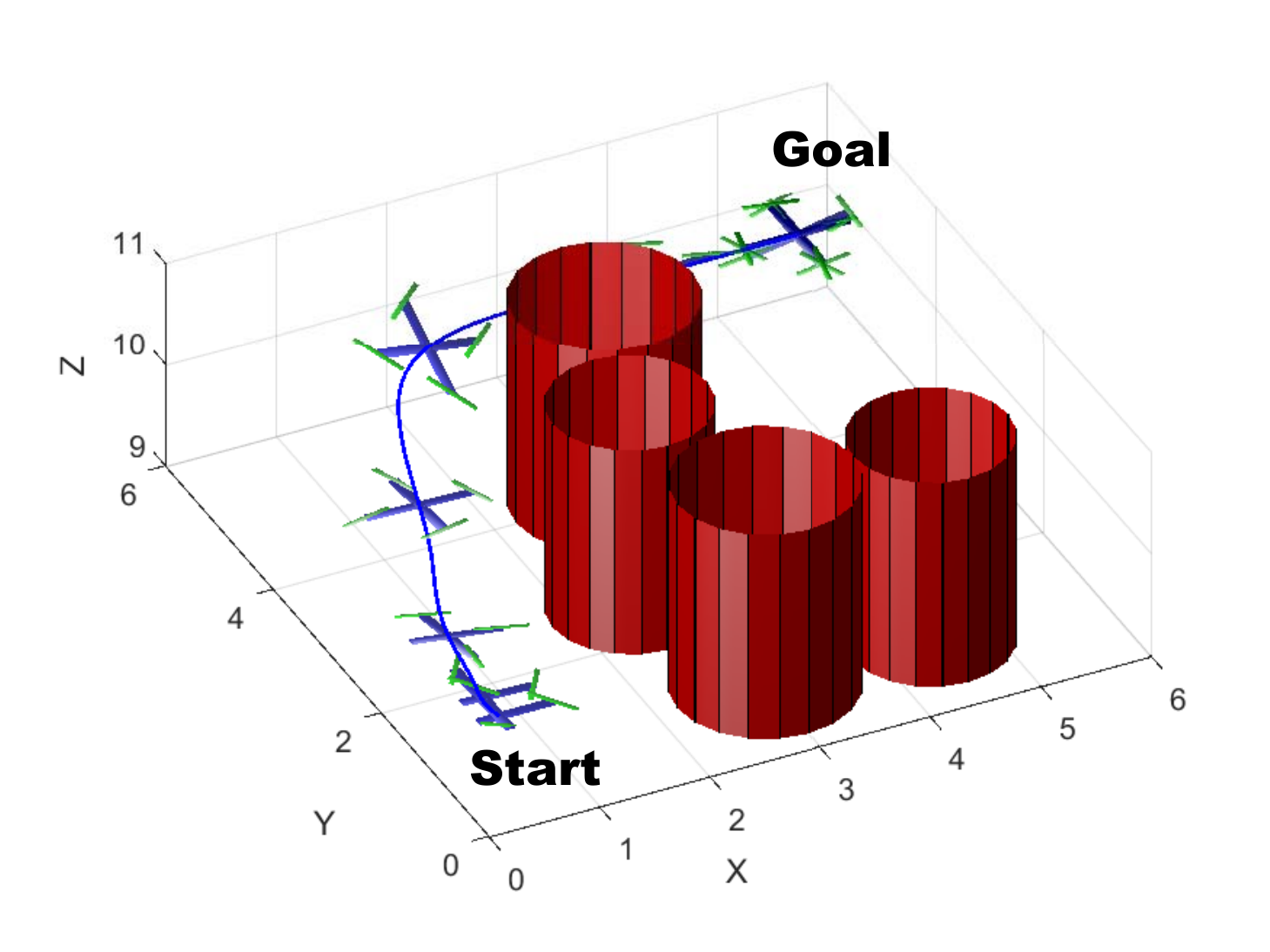}}
	\subfigure[Task 2 with  $0$ level policy]{
		\includegraphics*[width=.475\columnwidth, viewport =65 20 370 280]{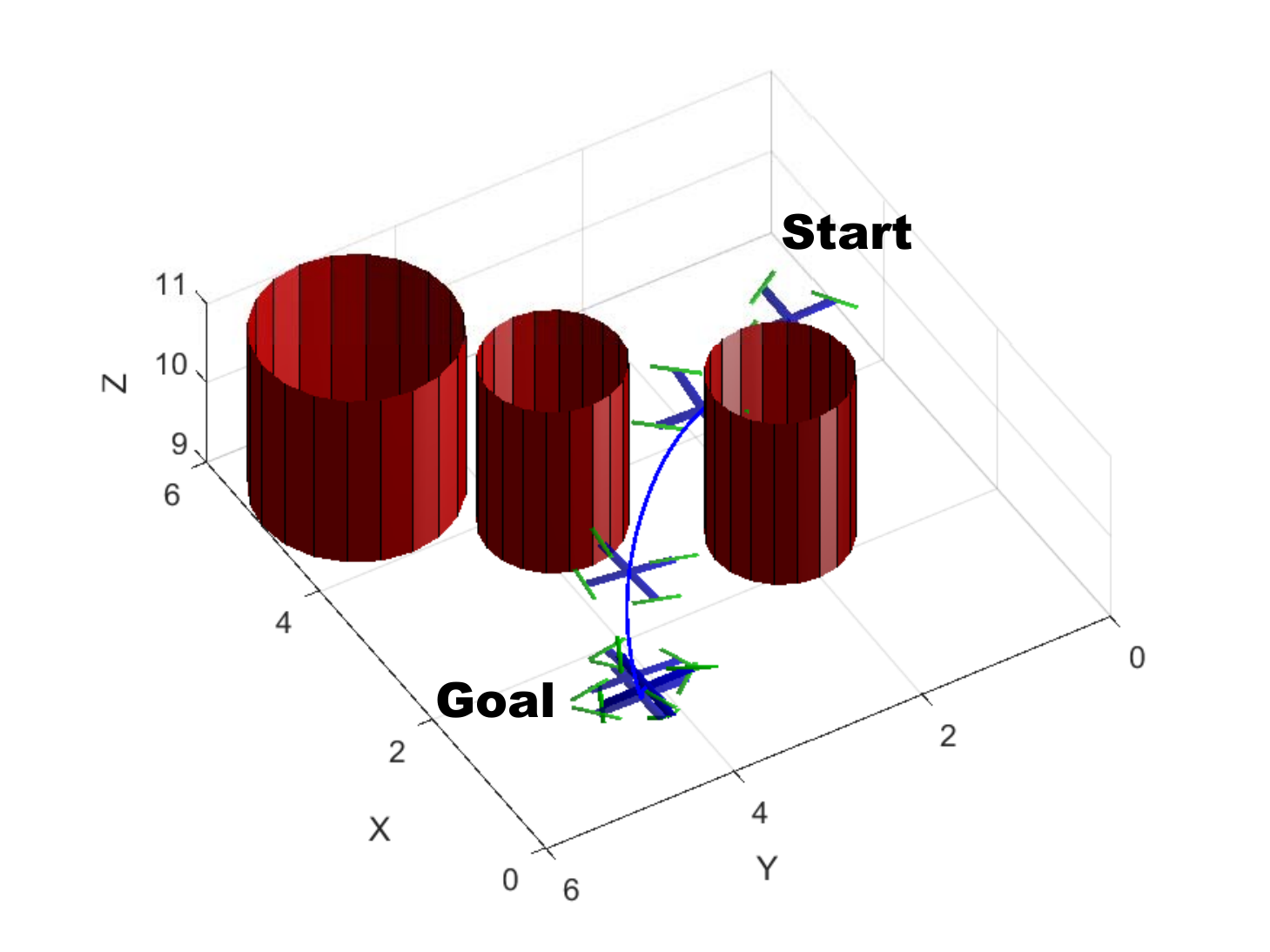}}
	\subfigure[Task 2 with  $4$ level policy]{
		\includegraphics*[width=.475\columnwidth, viewport =65 20 370 280]{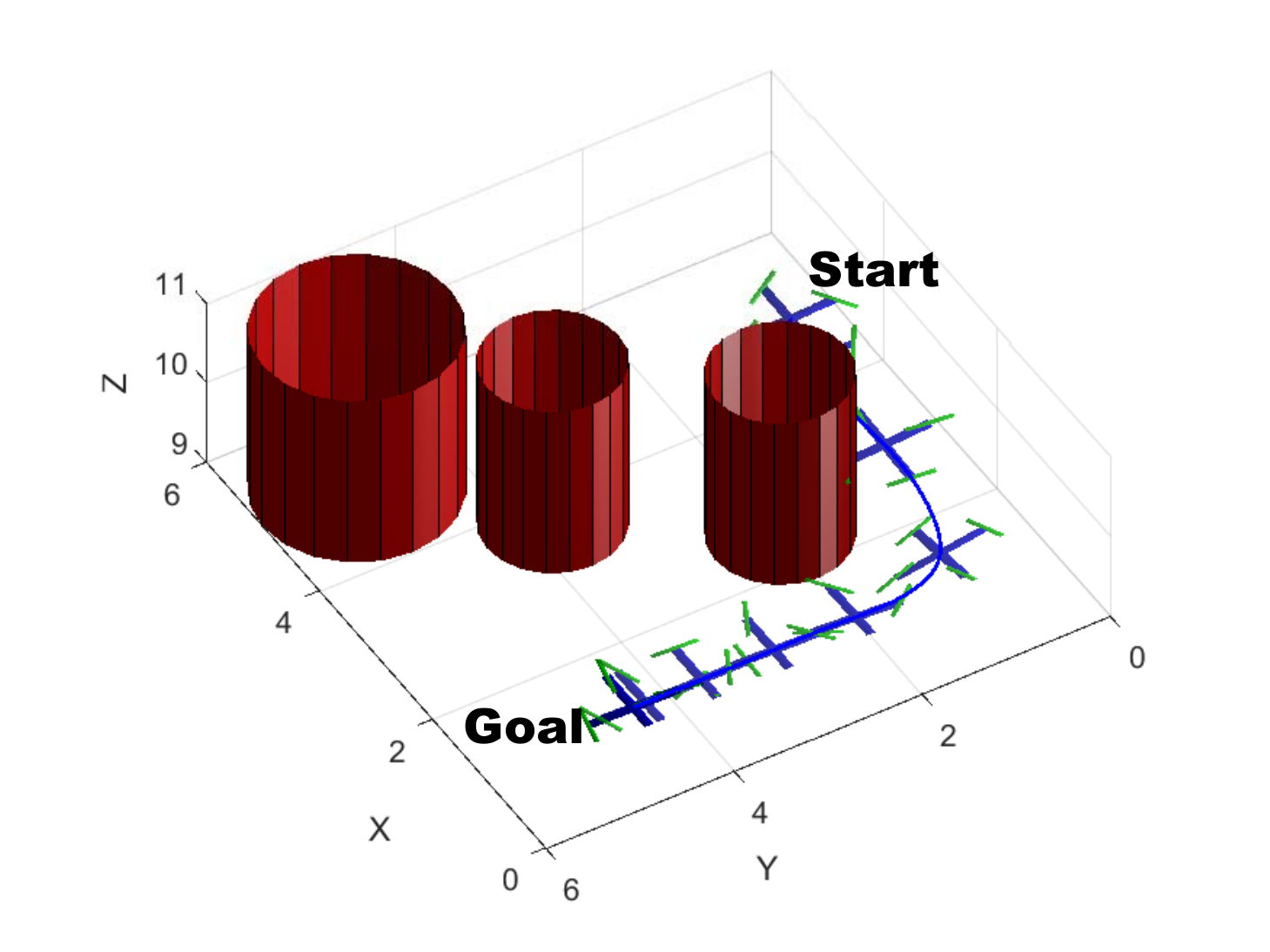}}
	\caption{Quadrotor control example. The cost functions were given so that the quadrotor reaches the goal region without colliding with obstacles.}\vspace*{-.3in}
	\label{fig:quad}
\end{figure}

\section{Conclusions}
This work has proposed a multiscale framework to solve a class of continuous-time, continuous-space stochastic optimal control problems in a complex environment.
Using bases obtained by the diffusion wavelet method, the problem has been solved efficiently: the global plan was computed with coarse resolution and a detailed plan was only obtained for important regions.
In addition, when combined with a receding-horizon scheme in the execution of the optimal control solution, the proposed method can generate a continuous control sequence for robot motion.
Numerical examples demonstrated that the optimal solution can be expressed in very low-dimensional parametric space since the multi-scale abstraction provides meaningful basis functions.

It is worth noting that, because a Markov chain only needs to contain information about dynamics and the domain geometry, the same abstraction can be re-utilized for various tasks.
We expect that this will allow us to treat more challenging planning and control problems where extensive robotic research is being investigated. 

\addtolength{\textheight}{-12cm}   



\small
\section*{Acknowledgment}
This work was supported by Agency for Defense Development (under in part  contract  \#UD140053JD and in part contract \#UD150047JD).
\vspace*{-.2in}


\bibliographystyle{IEEEtran}
\bibliography{icra17}

\end{document}